%% file: iclr2020_arXiv.tex
\documentclass{article} % For LaTeX2e
\usepackage{iclr2020_conference,times}

\input{math_commands.tex}

\usepackage{hyperref}
\usepackage{url}

\title{Universal approximations of permutation invariant/equivariant functions by deep neural networks}

\author{Akiyoshi Sannai \& Yuuki Takai  \\
RIKEN Center for Advanced Intelligence Project/Keio University\\
Tokyo, Japan \\
\texttt{\{akiyoshi.sannai, yuuki.takai\}@riken.jp} \\
\And
Matthieu Cordonnier  \\
\'{E}cole Normale Sup\'{e}rieure Paris-Saclay, \\ 
Cachan, France\\
\texttt{matthieu.cordonnier@ens-paris-saclay.fr} \\
}

\iclrfinalcopy % Uncomment for camera-ready version, but NOT for submission.
\begin{document}

\maketitle

\begin{abstract}

In this paper, we develop a theory about the relationship between 
$G$-invariant/equivariant functions and deep neural networks for finite group $G$.  
% Our main contribution is to prove an invariant/equivariant version of the universal approximation theorem concerning what we call $G$-invariant/equivariant deep neural networks for a finite group $G$. 
% A $G$-equivariant deep neural network consists of stacking standard single-layer neural networks $F_i\colon X_i \to X_{i+1}$ for which every $F_i$ is equivariant with respect to the actions of $S_n$.
% On the other hand, a $G$-invariant deep neural network is made up by stacking  equivariant neural networks plus some standard single-layer neural networks $Z_i\colon X_i \to X_{i+1}$ for which every $Z_i$ is invariant with respect to the actions of $G$. We establish the following theorem: 
Especially, for a given $G$-invariant/equivariant function, we construct 
its universal approximator by deep neural network whose layers 
equip $G$-actions and each affine transformations are $G$-equivariant/invariant. 
% $G$-invariant/equivariant deep neural networks are universal approximators for permutation invariant/equivariant functions. %Our models are natural generalizations of the ones introduced by  Zaheer et al. %\cite{deepsets}.   
Due to representation theory, we can show that this approximator has exponentially fewer free parameters than usual models. 
% Moreover, we show that the number of free parameters appearing in these models turns out to be exponentially fewer than the number of the ones in the usual models.
% By combining these results, we conclude that although the number of free parameters is much smaller than the one of the usual model, the invariant/equivariant models can approximate invariant/equivariant functions with arbitrary accuracy.
%This justifies why our models are appropriate for  invariant/equivariant problems.

\end{abstract}

\section{Introduction}

Deep neural networks have great success in many applications such as image recognition, speech recognition, natural language process and others as \cite{krizhevsky2012imagenet}, \cite{goodfellow2013multi}, \cite{wan2013regularization}, 
and \cite{silver2017mastering}. 
 A common strategy in their works is to construct larger and deeper networks. %When we use larger and deeper networks, one of the main obstructions to learn is the so called {\it  curse of dimensionality}, i.e. if the parameter, dimension of the models increase, then the required sample size and hence computational complexity for approximation increases exponentially. One idea to overcome the curse of dimensionality is to design models which fit to objectives.
 However, one of the main obstructions about using very deep and large networks for learning tasks is the so-called {\it curse of dimensionality}. Namely, if the parameters' dimension increase, so does the required sample size. Then, the computational complexity becomes exponentially higher. An idea to overcome this 
 %{curse of dimensionality} 
 is to design models with respect to the objective structure.

\cite{deepsets} designed a model adapted to machine learning tasks defined on sets, which are, from a  mathematical point of view, permutation invariant or equivariant tasks. They demonstrate surprisingly good applicability on their method on population statistic estimation, point cloud classification, set expansion, and outliers detection. Empirically speaking, their results are really significant. Many researchers studied invariant/equivariant networks as \cite{qi2017pointnet},  \cite{hartford2018deep}, \cite{kondor2018covariant}, \cite{maron2018invariant}, 
%\cite{maron2019universality}, 
\cite{bloem2019probabilistic}, \cite{pmlr-v80-kondor18a}, and so on.
%However, its theoretical guarantee is not considered sufficiently in their paper. Hence, because of this motivation,  we generalize their models and give {\it invariant/equivariant versions} of the universal approximation theorem in this paper. Our theorem is not only for the classical {\it depth-bounded} version but also the {\it width-bounded} version.
Nevertheless, theoretical guarantee of their methods is not sufficiently considered.  One of our motivations is to establish a theoretical guarantee. 
In this paper, we prove an {\it invariant/equivariant version} of the universal approximation theorem by constructing an approximator. 
% In our theorem, we take into consideration the {\it depth-bounded} version as well as the {\it width-bounded} version of the universal approximation theorem.
%Furthermore, we also calculate the number of free parameters appeared in our invariant/equivariant models. As a result, the number of free parameters appeared in our models is exponentially smaller than the one of the usual models.
%Hence, we conclude that although the free parameters of the invariant/equivarint models are exponentially fewer than the one of the usual models, the invariant/equivariant models can approximate the invariant/equivariant functions to arbitrary accuracy.https://www.overleaf.com/project/5c5aa8c7238b3326afbeb8be
%This gives us an understanding of why the invariant/equivariant models designed in \cite{deepsets} work well.
%To show this, we construct an approximator. %Unfortunately, this approximator is different from the model of \citet{deepsets}. 
For the symmetric group, our approximator is close to the equivariant model of \citet{deepsets} in a sense (see a remark after Theorem \ref{theo:main-theorem}).  
We can calculate the number of the free parameters appearing in our invariant/equivariant model, and show that this number is exponentially smaller than one of the usual models. 
%Despite this, our universal approximation theorem implies that our invariant/equivariant models of networks are still able to approximate any invariant/equivariant functions defined on a compact set with arbitrary accuracy.

For usual deep neural networks, a {\it universal approximation theorem} was first proved by \cite{cyb}. It states that, when the width goes to infinity, a (usual) neural network with a single hidden layer can, with arbitrary accuracy, approximate any continuous function with compact support. Though his theorem was only for sigmoid activation functions, there are further versions of this theorem which allows some wider classes of activation functions. In the recent literature, the most commonly used activation function is the {\it ReLU (Rectified Linear Unit)} function, which is the one we focus on in this paper. Some important previous works  
on universal approximation theorem for {\it ReLU} activation function are by
\cite{barron1994approximation}, 
\cite{hornik1989multilayer}, 
\cite{funahashi1989approximate}, 
\cite{kurkova2002comparison}, and 
\cite{sonoda2017neural}. 
In particular, for a part of the proof of our main theorem, we borrow the results 
of \cite{sonoda2017neural} and \cite{approxrelu}.  
% [Barron, 1994; Hornik et al., 1989;  Funahashi, 1989; Kurkov$\acute{\mbox{a}}$]. 
The interest of the universal approximation theorem in learning theory is to guarantee that {\it we can search in the space which contains the solutions}. The universal approximation theorem states the existence of the model which approximates the target function in arbitrary accuracy. This means that if we use the suitable algorithm, we have the desired solutions. We cannot guarantee such situations without the universal approximation theorem. 
Our universal approximation theorem allows us to apply 
representation theory. By this point of view, we can calculate 
the number of free parameters of our approximator. 
%Hence, our invariant/equivariant version of the universal approximation theorem is important.

%However, the invariant models they construct have huge width in general and they asked whether there exist more efficient $G$-invariant models. As we see later, our universal approximation theorem gives an answer to their question in the case of $G=S_n$.

In the equivariant case, a technical key point of the proof is {\it the one to one correspondence} between $G$-equivariant functions and $\Stab_G(i)$-invariant functions. Here, $G$ is a finite group and $\Stab_G(i)$ is the subgroup of $G$  consisting of the elements which fix $i$. We first confirm this correspondence at the function level. After that, we rephrase it by deep neural networks. This correspondence enables us to reduce the equivariant case to the invariant case. 

The invariant case has already established by some researchers \cite{deepsets}, 
\cite{yarotsky2018universal}, \cite{maron2019universality}. 
 For $G=S_n$, here $S_n$ is the symmetric group of degree $n$, 
 \cite{deepsets} showed that 
{ a representation theorem} of $S_n$-invariant function 
which is famous as a solution for the Hilbert's 13th problem by \cite{kolmogorovrepresentation} and  \cite{arnoldrepresentation} %[Kolmogorov, 1961 Arnold, 1963]. 
gives us an explicit description. 
Due to this theorem and the usual universal approximation theorem, we can construct a concrete deep neural network of the invariant model. 
% Combining this with the reduction to the invariant case, we complete the proof of the equivariant case. 
Recently, \cite{maron2019universality} proved an invariant version of the universal approximation theorem for any finite group $G$ using tensor structures. 
We borrow their results to obtain our main results. 

% % There are two important invariants in deep neural networks. the first one is the width and the second one is the depth.
% To calculate the number of the free parameters of our approximator, we need to consider a depth-bound or a weight-bound of our model. 
% Considering these invariants, the universal approximation theorem of \cite{sonoda2017neural} can be viewed as a {\it depth-bounded} universal approximation theorem. A {\it width-bounded} universal approximation theorem is proved by \cite{approxrelu}, which states that even if the maximal width of deep neural networks is bounded, if the minimal width is bigger than the input dimension, then the universal approximation theorem holds. This kind of the universal approximation theorem is also important to question of neural network expressivity, i.e., how the architectural properties of a neural network (depth, width, layer type) affect the resulting functions it can compute, and its ensuing performance. 
% % We provide both of  invariant/equivariant versions of {\it depth-bounded} universal approximation theorem and {\it width-bounded} universal approximation theorem in this paper. \\[-20pt]
% We give both of the 
% depth-bounded and width-bounded versions of the invariant/equivariant universal approximation theorem in this  paper. 
% %The {\it width-bounded} version of our invariant universal approximation theorem is an answer to the question asked by \cite{maron2019universality}\\[-20pt]

\subsection{Contributions}
\vspace{-10pt}
Our contributions are summarized as follows:

$\bullet$ We prove an invariant/equivariant version of the approximation theorems, which is a one step to understand the behavior of deep neural networks with permutations or more generally group actions. 
%Furthermore, we give a depth/width bound of our approximators. 
% both of the depth-bounded and width-bounded versions of the invariant/equivariant universal approximation theorem. 

$\bullet$ Using representation theory, we calculate the number of free parameters appearing in our models. As a result, the number of parameter in our models is exponentially smaller than the one of the usual models. This means that our models are easier to train than the usual models.

$\bullet$ Although our model is slightly different from the equivariant  model of \cite{deepsets} for $G=S_n$,  
our theorem guarantees that our model for finite group $G$ can approximate any $G$-invariant/equivariant functions. 

%$\bullet$ As an application of our results, these two results above give an understanding of the models designed in \cite{deepsets}, i.e., our approximation theorem and the calculation of the free parameters are evidences for the models in \cite{deepsets} to work well.

\subsection{Related works}
Group theory, or symmetry is an important concept in mathematics, physics, and machine learning. In machine learning, deep symmetry networks (symnets) is designed by %[Gans, et al., 2014]  
\cite{Gens2014} as a generalization
of convnets that forms feature maps over arbitrary symmetry groups. Group equivariant Convolutional
Neural Networks (G-CNNs) is designed by %[Cohen, Weiling, 2016]
\cite{Cohen-Welling2016}, as a natural generalization of convolutional neural networks that reduces sample complexity by exploiting symmetries. The models for permutation invariant/equivariant tasks are designed by \cite{deepsets} to give great results on population statistic estimation, point cloud classification, set expansion, and outlier detection.

The universal approximation theorem is one of the most classical mathematical theorems of neural networks.  %The universal approximation theorem states that a feed-forward network with a single hidden layer containing a finite number of neurons can approximate continuous functions on compact subsets of $\mathbb{R}^n$.
As we saw in the introduction, \cite{cyb} proved this theorem in 1989 for sigmoid activation functions. After his achievement, some researchers showed similar results to generalize the sigmoid function to a larger class of activation functions %[Barron, 1994; Hornik et al., 1989;  Funahashi, 1989; Kurkov$\acute{\mbox{a}}$]
as 
\cite{barron1994approximation}, \cite{hornik1989multilayer}, \cite{funahashi1989approximate}, \cite{kurkova1992} and \cite{sonoda2017neural}.

As mentioned above, the invariant case has been established. 
For $G=S_n$, \cite{deepsets} essentially proved an invariant version of an 
universal approximation theorem. 
\cite{yarotsky2018universal} gave a more explicit $S_n$-invariant approximator by a shallow deep neural network. 
%which focused on the action of compact groups and provided network's architectures based on polynomial layer.  
\cite{maron2019universality} considered a $G$-invariant model with some 
tensor structures for any finite group $G$. 
An equivariant version 
%of universal approximation theorem 
for finite group $G$ by  
shallow (hyper-)graph neural networks is proved by \cite{keriven2019universal}. 
Our architecture of approximator is different from theirs. Moreover, 
although they proved only for ``squashing functions'' which exclude ReLU functions, our theorem allows us to use the ReLU functions. We also remark that 
our setting in this paper is quite general. In particular, ours include tensor structures, hence graph neural networks. 
It must be interesting to compare the numbers of free parameters of models of us and \cite{keriven2019universal}.  

% Thereby, they proved \textit{universal approximation theorem of invariant networks} for compact and finite group actions. Our present work contains three main results. Despite than our result about \textit{universality of invariant networks} overlaps with the contributions of \cite{maron2019universality} and \cite{yarotsky2018universal}, these latter don't provide any bound to control the width of their network's architectures. In our present work, we give such bounds for width as well as for depth of our network's architecture. 
% %\todonum{Add about papers pointed out}

\section{Preliminaries and main results}\label{sec:preliminaries}
In this paper, we treat fully connected deep neural networks.
We mainly consider ReLU activation functions. Here, 
the ReLU activation function is defined by 
\[
 \mbox{ReLU}(x)=\mbox{max}(0,x).  %=[x]_+.
\]
We remark that our argument during this paper works for any activation functions 
which satisfy a usual universal approximation theorem. 
% We remark that if the usual universal approximation theorem holds for some other activation functions, 
% then our proof works for these functions. 
%as  [Barron, 1994; Hornik et al., 1989;  Funahashi, 1989; Kurkov$\acute{\mbox{a}}$]. 
A deep neural network is built by stacking the blocks which consist of a linear map and a ReLU activation. More formally, it is a function $Z_i$ from $\mathbb{R}^{d_i}$ to $\mathbb{R}^{d_{i+1}}$ defined by $Z_i(\bm{x}) =  \mbox{ReLU}(W_i\bm{x}+ \bm{b}_i)$, where $W_i \in \mathbb{R}^{d_{i+1}\times d_i}, \bm{b}_i \in \mathbb{R}^{d_i+1}$. In this case, $d_i$ is called the width of the $i$-th layer. The output of the deep neural networks is 
\begin{align}
 Y(\bm{x})= Z_H \circ Z_{H-1} \ldots Z_2 \circ Z_1(\bm{x}), \label{eq:DNN}
\end{align}
where $H$ is called the depth of the deep neural network. 
We define the width of a deep neural network as the maximum of 
the widths of all layers. 
% Our models, which are generalization of the models in \cite{deepsets}, are defined by the invariant/equivariant property of $Z_i$. 
% Before defining the invariant/equivariant models, we define 
% the invariance/equivariance for functions. 

%Only two groups will be considered in this paper.  
Our main objects are deep neural networks which are invariant/equivariant 
with actions by a finite group $G$. 
We review some facts about groups and these actions here. 
Some more details are written in Appendix \ref{sec:appendix-groups}. 
Let $S_n$ be the group consisting of permutations of 
$n$ elements $\{1, 2, \dots, n\}$. 
This $S_n$ is called the symmetric group of degree $n$.   
 The symmetric group $S_n$ acts on 
$\{1, 2, \dots, n \}$ by the permutation $i \mapsto \sigma^{-1}(i)$ 
for $\sigma \in S_n$. 
By Proposition \ref{prop:embedding-to-symmetric-group} that 
any finite group $G$ can be regarded as a subgroup of $S_n$ 
for a positive integer $n$.
% For a finite group $G$ which is a subgroup 
% of $S_n$, 
Then, $G$ also acts on $\{1, \dots, n \}$ by the action as 
an element of $S_n$. For $i \in \{1, 2, \dots, n \}$, we define the 
orbit of $i$ as $O_i = G(i) = \{ \sigma(i) \mid \sigma \in G\}$. 
Then, the set $\{1,2, \dots, n \}$ can be divided to a disjoint union of 
the orbits: 
$\{1,2,\dots, n \} = \bigsqcup_{j=1}^m O_{i_j}$. 

Let $G$ be a finite group action on $\{1,2,\dots,n\}$.  
For $i =1,2,\dots, n$, we define the stabilizer subgroup $\Stab_G(i)$ of $G$ associated 
with $x$ by the subgroup of elements of $G$ fixing $i$. 
%\[
%\Stab_G(x) = \left\{\sigma\in G \mid \sigma\cdot x = x \right\}.
%\]
Then, by Proposition \ref{prop:orbit-stab-bijectivity}, 
the orbit $O_i$ and the set of cosets  
$G/\Stab_G(i)$ are bijective. 
When $G=S_n$, $X = \{1,2,\dots, n \}$, and $x = 1$, 
we set $\Stab_n(1) = \Stab(1) := \Stab_G(x)$.

We next consider an action of $S_n$ on the vector space $\RR^n$. 
% esspecially, each layers of deep neural networks. 
%the natural permutation of nodes in the fixed layer. 
The left action ``$\cdot$'' of $S_n$ on $\RR^n$ is defined by 
\[
 \sigma\cdot \bm{x} = \sigma\cdot (x_1, x_2, \dots, x_n)^\top 
 = (x_{\sigma^{-1}(1)}, x_{\sigma^{-1}(2)}, \dots, x_{\sigma^{-1}(n)})^\top
\]
for $\sigma \in S_n$ and $\bm{x} = (x_1, \dots, x_n)^\top \in \RR^n$.
We also call this the permutation action of $S_n$ on $\RR^n$. 
If there is an injective group homomorphism 
$\varphi\colon G \hookrightarrow S_n$, 
$G$ acts on $\RR^n$ by the permutation action as an element of 
$S_n$ by $\varphi$: For $\sigma \in G$ and $\bm{x}\in \RR^n$, 
we define $ \sigma\cdot \bm{x} := \varphi(\sigma)\cdot \bm{x}$. 
Then, we simply say that $G$ acts on $\RR^n$. 
%When we say that $G$ acts on $\RR^n$, we give an injective 

\begin{exmp}
The finite group $G = S_2$ is embedded into $S_3$ by 
$\varphi\colon (1\ 2) \mapsto (1 \ 2)$,  
where $(i \ j)$ is transposition between $i$ and $j$. 
In this case, the orbit decomposition of $\{1,2,3 \}$ is 
$\{ 1,2 \} \sqcup \{3\}$. 
By this embedding $\varphi$, an $S_2$-action on $\RR^3$ is defined: 
$\varphi((1 \ 2))\cdot(x_1, x_2, x_3)^\top = (x_2, x_1, x_3)^\top$. 
\end{exmp}

\begin{exmp}[Tensors]\label{ex:tensor}
The group action of $G$ which is a subgroup of $S_n$ on tensors as in \cite{maron2019universality} is realized as follows:  
An $S_n$-action on $\RR^{n^k\times a}$ is defined by the following 
injective homomorphism 
$\varphi\colon G \hookrightarrow S_{an^k}$: 
We fix a bijection from $\{1, \dots, a n^k \}$ to 
$\{ 1, \dots, n\}^k \times \{1,2,\dots, a \}$, and for $\sigma \in S_n$, $\varphi(\sigma) \in S_{an^k}$ 
is defined by 
\[
 \varphi(\sigma)\cdot (i_1, \dots, i_k, j) = 
 (\sigma^{-1}(i_1), \dots, \sigma^{-1}(i_k), j) 
\]
for $(i_1, \dots, i_k) \in \{1, \dots, n \}^k$ and $j \in \{1,2,\dots, a \}$. 
Then, for a tensor 
$X = (x_{i_1,\dots,i_k, j})_{i_1, \dots, i_k = 1\dots, n, \ j = 1,2,\dots,a } \in \RR^{n^k\times a}$, 
$\sigma \in S_n$ acts on $\RR^{n^k\times a}$ by  
$\varphi(\sigma)\cdot X = (x_{\sigma^{-1}(i_1),\dots,\sigma^{-1}(i_k),j})$. This action is same as one of \cite{maron2019universality}. 
\end{exmp}
\begin{exmp}[$n$-tuple of $D$-dimensional vectors]
We identify $\{1, 2, \dots, nD \}$ with 
$\{ (i,j) \mid i=1,\dots, n, j= 1, \dots, D\}$ and 
define $\varphi\colon S_n \hookrightarrow S_{nD}$ as 
$\varphi(\sigma)\cdot(i,j) = (\sigma^{-1}(i), j)$. 
Let $(\bm{x}_1, \dots, \bm{x}_n) \in (\RR^D)^n$ be an $n$-tuple of 
$D$-dimensional vectors. 
Then, for $\sigma \in S_n$, 
$\phi(\sigma)\cdot(\bm{x}_1, \dots, \bm{x}_n) = (\bm{x}_{\sigma^{-1}(1)}, \dots, \bm{x}_{\sigma^{-1}(n)})$. This means a permutation of $n$ vectors. 
\end{exmp}

\begin{defn}
Let $G$ be a finite group.  
We assume that an injective homomorphisms $\varphi\colon G \hookrightarrow S_m$ are given. Then, $G$ acts on $\RR^m$. 
%We assume that $G$ acts on $X$ (resp. $Y$) by $g\cdot x$ (resp. $g\ast y$) for $g\in G$ and  $x \in X$ (resp. $y \in Y$) .
We say that a map $f\colon \RR^m \to \RR^n$ is {\em $G$-invariant} if $f(\varphi(\sigma) \cdot \bm{x})=f(\bm{x})$ for any $ \sigma \in G$ and any $\bm{x} \in \RR^m$. We also assume that $\psi\colon G \hookrightarrow S_n$. 
Then, we say that a map $f\colon \RR^m \to \RR^n$ 
{\em $G$-equivariant} if  $f(\varphi(\sigma)\cdot \bm{x})= \psi(\sigma)\cdot f(\bm{x})$ for any $ \sigma\in G$. 

% \begin{itemize}
%     \item {\em $G$-invariant} if $f(\varphi(\sigma) \cdot \bm{x})=f(\bm{x})$ 
%     for any $ \sigma \in G$ and any $\bm{x} \in \RR^m$,
% \item {\em $G$-equivariant} if  $f(\varphi(\sigma)\cdot \bm{x})= \psi(\sigma)\cdot f(\bm{x})$ for any $ \sigma\in G$
% and any $\bm{x} \in \RR^m$.  
% \end{itemize}
% $\psi\colon G \hookrightarrow S_n$ 
When $G = S_n$ and the actions are induced by permutation,  
we call $G$-invariant (resp. $G$-equivariant) functions as  
{\em permutation invariant} (resp. {\em permutation equivariant}) functions.
\end{defn}
% \begin{defn}
% Let $S_n$ be the group of permutations of $n$ elements. Let $X$ and $Y$ be two sets. We assume that $S_n$ acts on $X$ (resp. $Y$) by $\sigma\cdot x$ (resp. $\sigma\ast y$) for $\sigma\in S_n$ and  $x \in X$ (resp. $y \in Y$) .
% We say that a map $f\colon X \to Y$ is 
% \begin{itemize}
%     \item {\em $S_n$-invariant} if $f(\sigma\cdot x)=f(x)$ 
%     for any $ \sigma \in S_n$ and any $x \in X$,
% \item {\em $S_n$-equivariant} if  $f(\sigma\cdot x)=\sigma*f(x)$ for any $ \sigma \in S_n$
% and any $x \in X$.  
% \end{itemize}
% \end{defn}

%Taking the inverse of $\sigma$ is the reason why we consider only the left action. 
% Throughout this paper, we abbreviate the vector $\bm{x} = (x_1, \dots, x_n)^\top \in \RR^n$ as $\bm{x} = (x_i)_i$.  
% Using this notation, we can rewrite the above induced action 
% as $\sigma\cdot (x_i)_i = (x_{\sigma^{-1}\cdot i})_i$.
%\label{eq:natural-action}$
% Like this action, if the action is caused by the permutation of the index set, i.e., 
% $ \sigma\cdot (x_i)_i = (x_{\sigma \cdot i})_i \label{eq:natural-action}$ holds , we say that $\sigma$ is the $S_n$-action. 

We define $G$-invariance and $G$-equivariance for deep neural networks. 
We can easily confirm that the models in \cite{deepsets} satisfies these properties. 

\begin{defn}
We say that a deep neural network 
$Z_H \circ Z_{H-1} \ldots Z_2 \circ Z_1$ as (\ref{eq:DNN}) is 
$G$-{\it equivariant} 
%(resp. {\it $S_n$-preinvariant}) 
if an action of $G$ on each of layers $\mathbb{R}^{d_i}$ is given by embedding $G\hookrightarrow S_{d_i}$ and the all corresponding map $Z_i\colon \mathbb{R}^{d_i} \to \mathbb{R}^{d_{i+1}}$ is $G$-equivariant. 
%(resp. $S_n$-invariant). 
We say that a deep neural network is $G$-{\it invariant} 
if there is a positive integer $c\leq H$ such that $G$-actions on each layer $\mathbb{R}^{d_i}$ for $1 \leq i \leq c+1$ are given and the corresponding map $Z_i\colon \mathbb{R}^{d_i} \to \mathbb{R}^{d_{i+1}}$ is $G$-equivariant for $1 \leq i \leq c-1$ and the map $Z_{c}:\RR^{d_{c}} \to \RR^{d_{c+1}}$ is $G$-invariant. 
%and $Z_{i}$ is a perceptron for  $c+2 \leq i \leq H$.
% Let $G$ be the group and $X$ and $Y$ two sets.
% We assume that $G$ acts on $X$ (resp. $Y$) by $g\cdot x$ (resp. $g\ast y$) for $g\in G$ and  $x \in X$ (resp. $y \in Y$) .
% We say that a map $f\colon X \to Y$ is 
% \begin{itemize}
%     \item {\em $G$-invariant} if $f(g\cdot x)=f(x)$ 
%     for any $ g \in G$ and any $x \in X$,
% \item {\em $G$-equivariant} if  $f(g\cdot x)= g*f(x)$ for any $ g\in G$
% and any $x \in X$.  
% \end{itemize}
\end{defn}

%In this paper, we give three main theorems about invariant/equivariant networks. The first one is the invariant version of universal approximation theorem.
Some approximation theorems for invariant functions have been 
already known: 
\begin{prop}[$G$-invariant version of universal approximation theorem]\label{theo:inv-theorem}
Let %\todo{$K$ should be $G$-stable? Is it too nervous?}
$G$ be a finite group which is a subgroup of $S_n$. 
Let $K$ be a compact set in $\mathbb{R}^n$ which is stable for the corresponding $G$-action in $\mathbb{R}^n$. Then, for any $f\colon K\rightarrow \mathbb{R}^m$ which is continuous and $G$-invariant and for any $\epsilon>0$, the following $G$-invariant ReLU neural networks  
$\mathcal{N}^{\inv}_G\colon \RR^n \to \RR^m$ satisfy that 
these represented functions $R_{\mathcal{N}^\inv_G}$ satisfy  
$\|f-R_{\mathcal{N}^\inv_G}\|_\infty < \epsilon$: 
\begin{itemize}
    \item $\mathcal{N}^\inv_G = \Sigma \circ L_H\circ \ReLU \circ \cdots \circ  L_1$, where $\Sigma$ is the summation ($G$-invariant part) and 
    $L_i\colon\RR^{{n^{d_i}}\times a_i}\to \RR^{n^{d_{i+1}}\times a_{i+1}}$ 
    is a linear map such that $L_i(g\cdot X) = g\cdot L_i(X)$ for any 
    $g\in G$ and any $X \in \RR^{{n^{d_i}}\times a_i}$. Here, the actions on each layers except for the output layer are same as Example \ref{ex:tensor}.  
    \item For $G = S_n$. 
    $\mathcal{N}^\inv_G = \mathcal{N}_\rho \circ 
    \Sigma \circ (\mathcal{N}_\phi, \dots, \mathcal{N}_\phi)^\top$, where 
    $\mathcal{N}_\rho$ (resp. $\mathcal{N}_\phi$) is a deep neural network approximating $\rho$ (resp. $\phi$) defined below. 
\end{itemize}
% Furthermore, we can take $\mathcal{N}$ as either of the following:

% $\bullet$ $\mathcal{N}$ has two hidden layers and the width is not bounded, or 

% $\bullet$ The width is of $\mathcal{N}$ is bounded above by $n(n+2)$ and the depth is not bounded.
\end{prop}

\begin{figure}[t]
\begin{center}
\scalebox{1}[1]{
\begin{tikzpicture}
        \node[functions] (rho) {$\rho$};
        
        \node[output,right=2em of rho] (output) {};
            \path[draw,->] (rho) -- (output);
        \node[node,left=1.5em of rho] (a) {}
        node[sum,left=1.5em of a] (sum) {$\sum$};
            \path[draw,->] (sum) -- (a);
            \path[draw,->] (a) -- (rho);
        \node[left=1.5em of sum] (2) {$\vdots$} -- (2) node[left=3em of 2] (l2) {$\vdots$} node[left=3em of l2] (ll2) {$\vdots$}; %node[below of=l2] (ldots) {$\vdots$} node[below of=ll2] (lldots) {$\vdots$} ;
        %\node[node,left=3em of sum] (2) {} -- (2) node[functions,left=2em of 2] (l2) {$\phi$} node[input,left=2em of l2] (ll2) {$x_2$} node[below of=l2] (ldots) {$\vdots$} node[below of=ll2] (lldots) {$\vdots$} ;
           % \path[draw,->] (ll2)--(l2);
           % \path[draw,->] (l2) -- (2);
            %\path[draw,->] (2) -- (sum);
       % \node[below of=2] (dots) {$\vdots$}   ;
        \node[node,below of=2] (n) {} -- (n) node[functions,left=2em of n] (ln) {$\phi$} node[input,left=2em of ln] (lln) {$x_n$} ;
            \path[draw,->] (lln)--(ln);
            \path[draw,->] (ln) -- (n);
            \path[draw,->] (n) -- (sum);
        \node[node,above =1em of 2] (1) {} -- (1) node[functions,left =2em of 1] (l1) {$\phi$}node[input,left=2em of l1] (ll1) {$x_2$} ;
            \path[draw,->] (ll1)--(l1);
            \path[draw,->] (l1) -- (1);
            \path[draw,->] (1) -- (sum);
        \node[node,above =1em of 1] (0) {} -- (0) node[functions,left=2em of 0] (l0) {$\phi$}node[input,left=2em of l0] (ll0) {$x_1$} ;
            \path[draw,->] (ll0)--(l0);
            \path[draw,->] (l0) -- (0);
            \path[draw,->] (0) -- (sum);
        %\node[above of=2 ,font=\scriptsize] {inputs};
        %\node[below of=n,font=\scriptsize] {weights};
    \end{tikzpicture}
}
\end{center}
    \caption{A neural network approximating 
    $S_n$-invariant function $f$. 
    In blue: the inputs, in red: the output, in green: $\rho$ and $\phi$ who have to be learned.}
    \label{fig:invariant}
\end{figure}
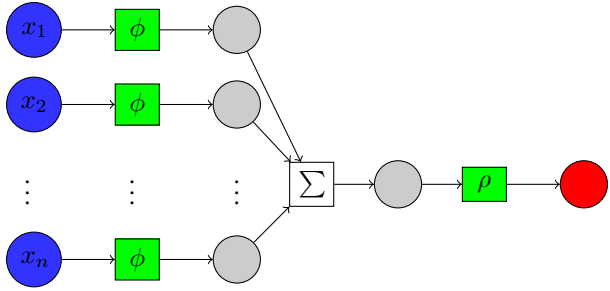

This proposition for $G = S_n$ is proven by \cite{deepsets},  \cite{yarotsky2018universal}. For general finite group $G$, \cite{maron2019universality} proved it.   
Diagram \ref{fig:invariant} illustrates the $S_n$-invariant ReLU neural network appeared in Proposition \ref{theo:inv-theorem}. 
%In Maron's paper, they 
%\cite{maron2019universality} also showed this theorem.  
%However,  we give the bound of the width and the depth. 
The key ingredient of the proof by \cite{deepsets} is the following 
Kolmogorov-Arnold representation theorem: 
\begin{theo}[\cite{deepsets} Kolmogorov-Arnold's representation theorem for permutation actions]\label{thm:representation-of-invariant}
Let $K\subset\mathbb{R}^n$ be a compact set. 
Then, any continuous $S_n$-invariant function $f\colon K\to\mathbb{R}$ can be represented as \vspace{-12pt}
\begin{align}
 f(x_1,\dots,x_n)=\rho\left(\sum_{i=1}^{n}\phi(x_i)\right) \label{eq:representation}
\end{align}
\vspace{-4pt}
for some continuous function $\rho\colon\mathbb{R}^{n+1}\rightarrow\mathbb{R}$. 
Here, $\phi\colon \mathbb{R} \to \mathbb{R}^{n+1}; x \mapsto 
(1, x, x^2, \dots, x^n)^\top$.
\end{theo}

%to give a presentation $f=\rho\left(\sum_{i=1}^{n}\phi(x_i)\right)$, where $f$ is the objective function. 
Since $\phi (x)$ has only one variable, we line up the copies of the network which approximates $\phi(x)$. Then, by combining $\Sigma$ and the network which approximates $\rho$, we obtain the network which approximates $f$. By using the theorem of \cite{approxrelu} (resp. \cite{sonoda2017neural}), we obtain the bound of the width (resp. the depth) for approximation of $\phi$ and $\rho$. \cite{maron2019universality} proved this proposition using a tensor structure. 

The main theorem is a $G$-equivariant version of universal approximation 
theorem. To state the main theorem, we need some notation. 
Let $G$ be a finite group acting on $\RR^n$. We set the orbit decomposition of $\{1, 2, \dots, n \}$ as $\{1, 2, \dots, n \} = \bigsqcup_{j = 1}^m O_{i_j}$,
and let $O_{i_j} = \{ i_{j1}=i_j, i_{j2},  \dots, i_{jl_j} \}$. Without loss of generality, we may reorder $\{1, 2, \dots, n \}$ as 
\[
 O_{i_1} = \{1, 2, \dots, l_1 \}, O_{i_2} = \{ l_1 + 1, l_1 +2, \dots l_1 + l_2  \}, \dots
\]
and $i_j = \sum_{k = 1}^{j-1} l_k + 1$ for $j = 1,2, \dots, m$.  
For each $j$, let 
$
 G = \bigsqcup_{k=1}^{l_j}\Stab_G(i_j) \tau_{j,k}
$
be the coset decomposition by $\Stab_G(i_j)$> Then, by ... we may assume that 
$\tau_{j,k}\in G$ satisfies $\tau_{j,k}^{-1}(i_j) = i_j + k$ for 
$k=1,2,\dots, l_j$. Then, the main theorem is the following: 
\begin{theo}[$G$-equivariant version of universal approximation theorem]\label{theo:main-theorem}
Let %\todo{$K$ should be $G$-stable? Is it too nervous?}
$G$ be a finite group which is a subgroup of $S_n$. 
Let $K$ be a compact set in $\mathbb{R}^n$ which is stable for the corresponding $G$-action in $\mathbb{R}^n$. Then, for any 
$F\colon K\rightarrow \mathbb{R}^n; \bm{x} \mapsto (f_1, \dots, f_n)$ 
which is continuous and $G$-equivariant 
and for any $\epsilon>0$, the following $G$-invariant ReLU neural network  
$\mathcal{N}^{\equi}_G\colon \RR^n \to \RR^n$ satisfies that 
these represented functions $R_{\mathcal{N}^\equi_G}$ satisfy  
$\|F-R_{\mathcal{N}^\equi_G}\|_\infty < \epsilon$: 
\[
 \mathcal{N}^\equi_G = (\mathcal{N}^\inv_{\Stab_G(i_1)}\circ \tau_{1,1}, \  
 \mathcal{N}^\inv_{\Stab_G(i_1)}\circ \tau_{1,2}, \ \dots, \ 
 \mathcal{N}^\inv_{\Stab_G(i_m)}\circ \tau_{m,l_m})^\top.   
\]
Here, $\mathcal{N}^\inv_{\Stab_G(i_j)}$ is the $\Stab_G(i_j)$-invariant deep neural network approximating $f_{i_j}$ 
as in Proposition \ref{theo:inv-theorem}, and the actions on each layers are defined as follows: Each of hidden layers are written by $\RR^{n^2}\otimes_\RR V$ for a vector space $V$. On this space, $\sigma \in G$ acts on 
$(\bm{x}_1, \dots, \bm{x}_n)\otimes v \in \RR^{n^2}\otimes_\RR V$ ($\bm{x}_i\in \RR^n$) by 
\[
 \sigma\cdot ((\bm{x}_1, \dots, \bm{x}_n)\otimes v) = 
 (\wt{\sigma}_{1,1}\cdot \bm{x}_{\sigma^{-1}(1)}, \dots, \wt{\sigma}_{m,l_m}\cdot\bm{x}_{\sigma^{-1}(n)})\otimes v,  
\]
where $\wt{\sigma}_{j,k}$ is the element of $\Stab_G(i_j)$ satisfying 
$\wt{\sigma}_{j,k} = \tau_{j,k'}\sigma \tau_{j,k}^{-1}$ for some $k'= 1,2,\dots, l_j$. 

%$\bullet$ $\mathcal{N}$ has two hidden layers and the width is not bounded, or
%
%$\bullet$ The width is of $\mathcal{N}$ is bounded above by $n^3$ and the depth is not bounded.
\end{theo}

%\begin{textblock*}{0.4\linewidth}(300pt,180pt)
\begin{figure}[t]
   % \centering
\begin{center}
\scalebox{1}[1]{
\begin{tikzpicture}%[baseline=20cm]
 \tikzset{basic/.style={draw,fill=blue!80,text width=0.75em,text badly centered}}       
        \node[functions] (rho) {$\rho$};
        
        \node[output,right=1em of rho] (output) {};
            \path[draw,->] (rho) -- (output);
        \node[node,left=1em of rho] (node1) {};
          \path[draw,->] (node1)--(rho);
        \node[sum,left=1em of node1] (sum) {$\!\sum$};
            \path[draw,->] (sum) -- (node1);
        \node[left=1em of sum] (2) {$\vdots$} -- (2) node[left=1.5em of 2] (l2) {$\vdots$} node[left=1.5em of l2] (ll2) {$\vdots$} node[above=1em of ll2] (a) {} node[below=1em of ll2] (b) {}  node[above=1em of 2] (a') {} node[below=1em of 2] (b') {};
             %\path[draw,->] (a)--(a');
            % \path[draw,->] (b)--(b');
             %\path[draw,->] (ll2)--(2);
            % \path[draw,->] (a')--(sum);
             % \path[draw,->] (b')--(sum);
               %\path[draw,->] (2)--(sum);
        \node[node,below=1.5em of 2] (n) {} -- (n) node[functions,left=1em of n] (ln) {$\phi$} node[node,left=1em of ln] (lln) {} node[node, left=1em of lln] (llln) {};
           \path[draw,->] (llln)--(lln);
            \path[draw,->] (lln)--(ln);
            \path[draw,->] (ln) -- (n);
            \path[draw,->] (n) -- (sum);
        \node[node,above =1em of 2] (1) {} -- (1) node[functions,left =1em of 1] (l1) {$\phi$}node[node,left=1em of l1] (ll1) {} node[node, left=1em of ll1] (lll1)  {};
           \path[draw,->] (lll1)--(ll1);
            \path[draw,->] (ll1)--(l1);
            \path[draw,->] (l1) -- (1);
            \path[draw,->] (1) -- (sum);
        \node[node,above =1em of 1] (0) {} -- (0) node[sum,left=1em of 0] (l0) {$\mathrm{id}$}node[node,left=1em of l0] (ll0) {} node[node, left=1em of ll0] (lll0) {};
           \path[draw,->] (lll0)--(ll0);
            \path[draw,->] (ll0)--(l0);
            \path[draw,->] (l0) -- (0);
\node[above=1.5em of llln] (f) {$\vdots$};
     \path[draw,->] (f)--(ll2);
\node[sum, above=3em of sum] (id) {$\mathrm{id}$}; 
\node[node, below right=1em of id] (a) {};
\path[draw,->] (0)--(id);
\path[draw,->] (id)--(a);
\path[draw,->] (a)--(rho);
\node[below=1.5em of ln] (dot) {$\vdots$};

\node[functions, below=15em of rho] (rhon) {$\rho$};
        
        \node[output,right=1em of rhon] (outputn) {};
            \path[draw,->] (rhon) -- (outputn);
        \node[node,left=1em of rhon] (noden) {};
            \path[draw,->] (noden)--(rhon);
        
        \node[sum,left=1em of noden] (sumn) {$\!\sum$};
            \path[draw,->] (sumn) -- (noden);
        \node[left=1em of sumn] (2n) {$\vdots$} -- (2n) node[left=1.5em of 2n] (l2n) {$\vdots$} node[left=1.5em of l2n] (ll2n) {$\vdots$} node[above=1em of ll2n] (an) {} node[below=1em of ll2n] (bn) {}  node[above=1em of 2n] (a'n) {} node[below=1em of 2n] (b'n) {};
             %\path[draw,->] (an)--(a'n);
             %\path[draw,->] (bn)--(b'n);
             %\path[draw,->] (ll2n)--(2n);
             %\path[draw,->] (a'n)--(sumn);
              %\path[draw,->] (b'n)--(sumn);
               %\path[draw,->] (2n)--(sumn);
        \node[node,below=1.5em of 2n] (nn) {} -- (nn) node[functions,left=1em of nn] (lnn) {$\phi$} node[node,left=1em of lnn] (llnn) {} node[node, left=1em of llnn] (lllnn) {};
            
            \path[draw,->] (llnn)--(lnn);
            \path[draw,->] (lnn) -- (nn);
            \path[draw,->] (nn) -- (sumn);
        \node[node,above =1.5em of 2n] (1n) {} -- (1n) node[functions,left =1em of 1n] (l1n) {$\phi$}node[node,left=1em of l1n] (ll1n) {} node[node, left=1em of ll1n] (lll1n) {};
            \path[draw,->] (lll1n)--(ll1n);
            \path[draw,->] (ll1n)--(l1n);
            \path[draw,->] (l1n) -- (1n);
            \path[draw,->] (1n) -- (sumn);
        \node[node,above =1em of 1n] (0n) {} -- (0n) node[sum,left=1em of 0n] (l0n) {$\mathrm{id}$}node[node,left=1em of l0n] (ll0n) {} node[node, left=1em of ll0n] (lll0n) {};
            \path[draw,->] (lllnn)--(ll0n);
            \path[draw,->] (lll0n)--(llnn);
            \path[draw,->] (ll0n)--(l0n);
            \path[draw,->] (l0n) -- (0n);
\node[above=1.5em of lllnn] (d) {$\vdots$};
     \path[draw,->] (d)--(ll2n);
\node[sum, above=3em of sumn] (idn) {$\mathrm{id}$}; 
\node[node, below right=1em of idn] (an) {};
\path[draw,->] (0n)--(idn);
\path[draw,->] (idn)--(an);
\path[draw,->] (an)--(rhon);
 
\node[below=1.7 of rho] (c) {$\vdots$};
   
\node[input, below left  =3em of lll1] (x1) {$x_1$} node[input,below=8em of x1] (xn) {$x_n$} node[below=3.7em of x1] (g) {$\vdots$} ;
 \path[draw,->] (x1)--(lll0);
% \path[draw,->] (x1)--(lll02);
 \path[draw,->] (x1)--(lll0n);
 \path[draw,->] (xn)--(llln);
 %\path[draw,->] (xn)--(llln2);
 \path[draw,->] (xn)--(lllnn); 
   
%\node[below of= lllnn , font=\scriptsize] {$\bigoplus_{i=1}^{n}\mathbb{R}^{n}$};  
%\node[below =22.5em of xn , font=\scriptsize] {$\mathbb{R}^{n}$}; 
%\node[below of= llnn , font=\scriptsize] (h) {$\bigoplus_{i=1}^{n}\mathbb{R}^{n}$ };
%\node[below of= h , font=\scriptsize] {and the layer just after, some space $\bigoplus_{i=1}^{n}V$ };
   
   \end{tikzpicture}
}
\end{center}
%\end{textblock*}
    \caption{A neural network approximating $S_n$-equivariant map $F$}
    \label{fig:Total-DNN}
\end{figure}
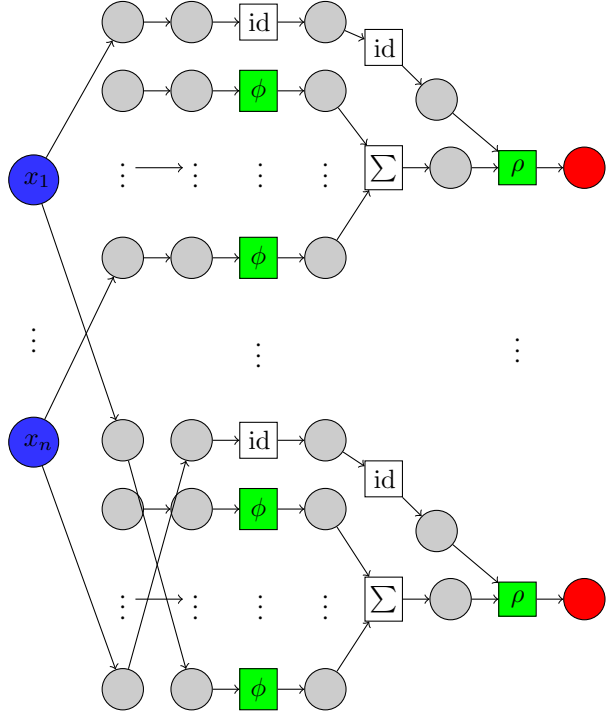

For $G=S_n$, when $\mathcal{N}^\equi_G$ is 
represented by $Z_H\circ Z_{H-1} \circ \cdots \circ Z_1$ as in (\ref{eq:DNN}), we can consider that $Z_2, \dots, Z_H$ are 
$S_n$-equivariant by usual permutation. In this sense, our model is 
close to the $S_n$-equivariant model in \cite{deepsets}. 

%Although invariant versions are proved by \cite{maron2019universality} and \cite{yarotsky2018universal}, the  equivariant version of universal approximation theorem for $S_n$ is first in the literature. 
Our strategy for the proof is the following:  At first, we establish the correspondence between $\Stab_G(i_j)$-invariant functions and $G$-equivariant functions. By this correspondence, we take  $\Stab_G(i_j)$-invariant function $f$ corresponding to the objective function $F$. By Proposition \ref{theo:inv-theorem}, we can approximate $f$ by a $\Stab_G(i_j)$-invariant network $\mathcal{N}^\inv_{\Stab_G(i_j)}$. Using 
$\mathcal{N}^\inv_{\Stab_G(i_j)}$, we construct the $G$-equivariant network which approximates $F$. Diagram \ref{fig:Total-DNN} illustrates the $S_n$-equivariant ReLU neural network appeared in Theorem~\ref{theo:main-theorem}. 

%\todo{Revise this paragraph}
Due to our universal approximation theorems, if the free parameters of the invariant/equivariant models are fewer than the ones of the usual models, we have a guarantee for using the invariant/equivariant models. The following definition illustrates the swapping of nodes.
\begin{defn}
Let $X=\{ 1, \dots, M\}$ be an index set of nodes in a layer. We say an $S_n$-action on $X$ is a union of permutations if  $X = \bigsqcup X_i$, where each $X_i$ has $n$ elements and $S_n$ acts on $X_i$ by permutation.%$S_n$ acts on each $X_i$ and the action is the permutation of $S_n$. 
 %(Hence, $X_i$ has $n$ points.)
\end{defn}
\begin{theo}\label{thm:number-of-params}
Let $\mathcal{N}$ be an $S_n$-invariant model of depth $D$ and width $M$ whose number of the equivariant layers is $d$ (resp. an $S_n$-equivariant model of depth $D$ and width $M$). Assume that the action is a union of permutations on nodes in each equivariant layer.
Then, the number of free parameters in this model is bounded by $M^{2D}\cdot (2/n^{2})^d$ (resp. $M^{2D}\cdot (2/n^2)^D$).
\end{theo}

Note that the number of free parameters in the usual model is $M^{2D}$. Hence, this theorem implies that {\it the free parameters of the invariant/equivariant models are exponentially fewer than the ones of the usual models}. 

\label{sec:headings}

\section{Equivariant case}\label{sec:equivariant}

In this section, we prove Theorem \ref{theo:main-theorem}, namely, the equivariant version of the universal approximation theorem. 
The key ingredient is the following proposition (proof is in Appendix \ref{sec:appendix-proof-of-prop}): 
\begin{prop}\label{prop:relation-invariant-equivariant}
Notations are same as Theorem \ref{theo:main-theorem}.
Then, a map $F\colon\RR^n \to \RR^n$ is $G$-equivariant if and only if $F$ can be 
represented by $F = (f_1\circ \tau_{1,1}, f_1\circ \tau_{1,2}, \dots, f_m\circ \tau_{m,l_m})^\top$ for some $\Stab_G(i_j)$-invariant functions 
$f_j\colon\RR^n \to \RR$. 
Here, $\tau_{j,k} \in G$ is regarded as a linear map $\RR^n\to \RR^n$. 
\end{prop}
%We prove this proposition in Appendix~\ref{sec:appendix-proof-of-prop}. 

For simplicity, we prove Theorem \ref{theo:main-theorem} only for $G=S_n$. 
We can show the general case by a similar argument. 
More precisely, we construct an $S_n$-equivariant deep neural network 
%made from stacking $S_n$-equivariant single layers 
approximating the given $S_n$-equivariant function. 
Similarly, we can prove this theorem for any finite group $G$. 
To show Theorem \ref{theo:main-theorem} for $G=S_n$, we divide the proof 
to four steps as follows:\\[-14pt]
\begin{enumerate}
    \item By Proposition \ref{prop:relation-invariant-equivariant} proved below, we reduce 
    the argument on $S_n$-equivariant map $F$ to the one of $\Stab(1)$-invariant function $f$. 
    \item Modifying Theorem \ref{thm:representation-of-invariant}, 
    we have a representation of $\Stab(1)$-invariant function $f$. 
    \item Using the above representation, we have a $\Stab(1)$-invariant
    deep neural net which approximates $f$ and construct
a deep neural network  approximating $F$.   
    \item We introduce a certain action of $S_n$ on $(\RR^n)^n$ which  appears the first hidden layer naturally and show the $S_n$-equivariance between the input layer and 
    the first hidden layer.
\end{enumerate}
% Because step~1 has already been explained in Section~\ref{sec:relation-inv-equ}, we start from step~2 in 
% the next section. 

%%%%%%%%%%%%%%
% Should be replace as below?
%%%%%%%%%%%%%%
% Because step~1 has already been explained in Section~\ref{sec:relation-inv-equ}, we start from step~2 in 
%  Section \ref{subsec:Rep-of-Stab(1)-inv}. 
% In Section~\ref{subsec:approx-Stab(1)-inv}, we explain step~3. Section~\ref{subsec:another-action} is devoted to step~4. 
% In Section~\ref{subsec:equivariance-of-first-gap}, we show step~5 
% and conclude the proof of Theorem \ref{theo:main-theorem}.  
%\todo{Should be more simple.}  

% We have already explained about step 1 in Section~\ref{sec:relation-inv-equ}. 
%Let $f$ be the 
%$\Stab(1)$-function on $\RR^n$ to $\RR$ corresponding to $F$
%by Propsition \ref{prop:relation-invariant-equivariant}. 
% \subsection{Representation of Stab(1)-invariant functions}\label{subsec:Rep-of-Stab(1)-inv}

We first investigate step 1. 
We recall that, during this section, 
we only consider the action of $S_n$ on $\RR^n$ 
induced from permutation $\sigma\cdot (x_1,\dots, x_n)^\top = (x_{\sigma^{-1}(1)}, \dots, x_{\sigma^{-1}(n)})^\top$. 
Then, we remark that the orbit of $1$ by this action is the 
total set $O_1 = \{1,2,\dots, n \}$ and the coset decomposition of 
$S_n$ by $\Stab(1)$ is $S_n = \bigsqcup_{i=1}^n\Stab(1)(1 \ i)$. 
Thus, we have the following: 
\begin{cor}\label{cor:relation-invariant-equivariant}
A map $F\colon\RR^n \to \RR^n$ is $S_n$-equivariant if and only if there is $\Stab(1)$-invariant function 
$f\colon\RR^n \to \RR$ satisfying   
$F = (f\circ (1 \ 1), f_1\circ (1 \ 2), \dots, f_m\circ (1 \ n))^\top$. Here, $(i \ j) \in S_n$ is the transposition between 
$i$ and $j$ and is regarded as 
a linear map $\RR^n \to \RR^n$. 
\end{cor}

% The stabilizer subgroup of $S_n$ with 
% respect to $\{1 \}$ is defined by 
% \[
%  \Stab_n(1) =\left\{\sigma\in S_n \mid \sigma\cdot 1 =1\right\}.  
% \]
% When there is no ambiguity, we denote it by $\Stab(1)$.
% Then, for any $\Stab(1)$-invariant function $f:\RR^n \to \RR$,
% $F= (f, f\circ(1 \ 2), \dots, f\circ (1 \ n))^\top : \RR^n \to
% \RR^n$ is $S_n$-equivariant.   

%%%%%%%%%%%%%%%%%%%%%%%%%%%%%%%%%%%%%%%%%%%%%%%%%%%%%%%%%%%
%%%% Sketch of the proof of Prop. 4.1 commented out 2019/5/21 (Return 2019/5/22)
%%%%%%%%%%%%%%%%%%%%%%%%%%%%%%%%%%%%%%%%%%%%%%%%%%%%%%%%%%%
%The proof of this proposition is based on 
%the coset decomposition of $S_n$ by $\Stab(1)$ as 
%$S_n = \bigsqcup_{i=1}^n \Stab(1)(1 \ i)$ and 
%the fact that transposition $(1 \ i)$, $i=1, \dots, n$ 
%generate $S_n$. 
%Using these facts and some calculation, it is not hard to 
%show the equivalence.  
%%Due to the limitaion of space, we omit the proof.
% The detail of the proof is in Appendix \ref{sec:appendix-proof-of-prop}. 

Next, we consider step 2. 
% In this section, we establish step 2, namely, 
% we show a representation theorem of $\Stab(1)$-invariant functions similar to Theorem~\ref{thm:representation-of-invariant}. 
The stabilizer subgroup $\Stab_n(1)$ is isomorphic to $S_{n-1}$ as a group by Lemma. %\todo{Add Lemma in Appendix}.
Hence, we can regard the $\Stab(1)$-invariant function $f\colon\RR^n \to \RR$ as 
an $S_{n-1}$-invariant function. 
This point of view allows us to apply  
Theorem \ref{thm:representation-of-invariant} to $f$. 
Hence, we have the following representation theorem of $\Stab(1)$-invariant functions as a corollary of Theorem \ref{thm:representation-of-invariant}. 

\begin{cor}[Representation of $\Stab(1)$-invariant function]\label{cor:representation-of-Stab(1)-inv}
Let $K\subset\mathbb{R}^n$ be a compact set, let $f\colon K \longrightarrow\mathbb{R}$ be a continuous and $\Stab(1)$-invariant function. Then, $f(\bm{x})$ can be represented 
as 
\[
%\forall (x_1,\dots,x_n)\in K, 
f(\bm{x}) = f(x_1,\dots,x_n)=\rho\left(x_1,\sum_{i=2}^{n}\phi(x_i)\right), 
\]
for some continuous function  
$\rho\colon \mathbb{R}^{n+1}\longrightarrow\mathbb{R}$. Here, 
$\phi\colon\RR \to \RR^n$ is similar as in Theorem \ref{thm:representation-of-invariant}. 
% $\begin{array}[t]{l|rcl}
% \phi : & \mathbb{R} & \longrightarrow & \mathbb{R}^{n} \\
%     & x & \longmapsto & \begin{pmatrix}
%     1\\x\\x^2\\ \vdots \\x^{n-1}
%     \end{pmatrix} \end{array}
% $.
\end{cor}
By this corollary, %~\ref{cor:representation-of-Stab(1)-inv}, 
we can represent the $\Stab(1)$-invariant function $f:\mathbb{R}^n\longmapsto\mathbb{R}$ as $f=\rho\circ L \circ\Phi$, where 
$\Phi: \RR^n \to \RR \times (\mathbb{R}^{n})^{n-1}$ and 
$L: \RR\times (\RR^n)^{n-1} \to \RR\times \RR^n$ are \\[-10pt]
\[
 \Phi(x_1, \dots, x_n) = (x_1, \phi(x_2), \dots, \phi(x_n)), \ \  
  L(x, \ (\bm{y}_1, \dots, \bm{y}_{n-1})) = \left(x, \sum_{i=1}^{n-1} \bm{y}_i  \right).
\]
% $\Phi: \RR^n \to \RR \times (\mathbb{R}^{n})^{n-1}$ is  
% \[
%  \Phi(x_1, \dots, x_n) = (x_1, \phi(x_2), \dots, \phi(x_n)), 
% \]
% and $L: \RR\times (\RR^n)^{n-1} \to \RR\times \RR^n$ is 
% \[
%  L(x, \ (\bm{y}_1, \dots, \bm{y}_{n-1})) = \left(x, \sum_{i=1}^{n-1} \bm{y}_i  \right). 
% \]
%\todo{to be more readable.}

% \begin{align*}
% \begin{array}[t]{l|rcl}
% \Phi : \mathbb{R}^n & \longrightarrow & \mathbb{R}\times(\mathbb{R}^{n})^{n-1} \\
%     & \begin{pmatrix}
%     x_1\\ \vdots\\x_n
%     \end{pmatrix} & \longmapsto & \begin{pmatrix}
%     x_1\\\phi(x_2)\\ \vdots \\\phi(x_n)
%     \end{pmatrix} \end{array} \\
%     \begin{array}[t]{l|rcl}
% L : & \mathbb{R}\times(\mathbb{R}^{n})^{n-1} & \longrightarrow & \mathbb{R}\times\mathbb{R}^{n} \\
%     & \begin{pmatrix}x\\y\end{pmatrix} & \longmapsto & \begin{pmatrix}
%     x\\ \sum_{i=1}^{n-1}y_i
%     \end{pmatrix} \end{array}
% \end{align*}
%is linear. And $\rho$ is as in Corollary 3.1 .
%\\
%\\

% \subsection{Universal approximation theorem for Stab(1)-invariant functions}
% \label{subsec:approx-Stab(1)-inv}
Then, we consider step 3, namely, the existence of $\Stab(1)$-invariant deep neural network approximating the  function $f$. 
After that, using this approximator, we construct a deep neural network approximating $S_n$-equivariant function $F$. 
%By the result on the universal approximation theorem by ReLU neural network in [Sonoda-Murata, 2017], \cite{approxrelu}, 
%\todo{what is the precise citation?}, 
By a slight modification of the invariant version of Proposition \ref{theo:inv-theorem} for $\Stab(1)$-invariant case, 
there exists a sequence of deep neural networks $\{ A_m \}_m$
(resp. $\{ B_m\}_m$) which converges to $\Phi$ (resp. $\rho$) uniformly.
Then, the sequence of deep neural networks $\{ B_m\circ L \circ A_m\}_m$ converges to $f=\rho\circ L \circ\Phi$ uniformly. 

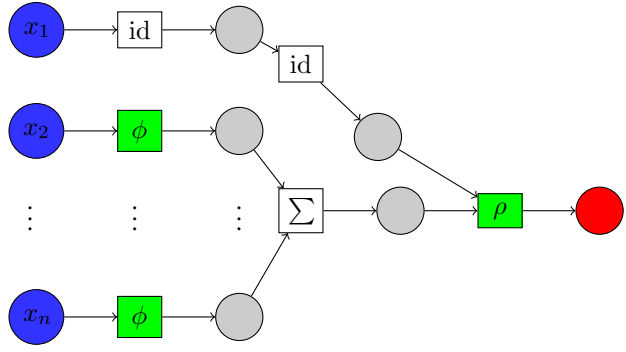
\begin{figure}[t]
\begin{center}
\scalebox{1}{
\begin{tikzpicture}
        \node[functions] (rho) {$\rho$};
        \node[node,left=2em of rho] (i) {};
        \node[output,right=2em of rho] (output) {};
            \path[draw,->] (rho) -- (output);
        \node[sum,left=2em of i] (sum) {$\sum$};
            \path[draw,->] (sum) -- (i);
            \path[draw,->] (i)--(rho);
        \node[left=1em of sum] (2) {$\vdots$} -- (2) node[left=3em of 2] (l2) {$\vdots$} node[left=3em of l2] (ll2) {$\vdots$} node[above=1em of ll2] (a) {} node[below=1em of ll2] (b) {}  node[above=1em of 2] (a') {} node[below=1em of 2] (b') {};
             %\path[draw,->] (a)--(a');
             %\path[draw,->] (b)--(b');
             %\path[draw,->] (ll2)--(2);
             %\path[draw,->] (a')--(sum);
             % \path[draw,->] (b')--(sum);
              % \path[draw,->] (2)--(sum);
        \node[node,below=2em of 2] (n) {} -- (n) node[functions,left=2em of n] (ln) {$\phi$} node[input,left=2em of ln] (lln) {$x_n$} ;
            \path[draw,->] (lln)--(ln);
            \path[draw,->] (ln) -- (n);
            \path[draw,->] (n) -- (sum);
        \node[node,above =1em of 2] (1) {} -- (1) node[functions,left =2em of 1] (l1) {$\phi$}node[input,left=2em of l1] (ll1) {$x_2$} ;
            \path[draw,->] (ll1)--(l1);
            \path[draw,->] (l1) -- (1);
            \path[draw,->] (1) -- (sum);
        \node[node,above =2em of 1] (0) {} -- (0) node[sum,left=2em of 0] (l0) {$\mathrm{id}$}node[input,left=2em of l0] (ll0) {$x_1$} ;
            \path[draw,->] (ll0)--(l0);
            \path[draw,->] (l0) -- (0);
\node[sum, above=4em of sum] (id) {$\mathrm{id}$}; 
\node[node, below right=2em of id] (a) {};
\path[draw,->] (0)--(id);
\path[draw,->] (id)--(a);
\path[draw,->] (a)--(rho);
       
       % \node[above of=ln,font=\scriptsize] {inputs};
        %\node[below of=n,font=\scriptsize] {weights};
    \end{tikzpicture}
}
\end{center}
    \caption{A neural network approximating the $\Stab(1)$-invariant function $f$}
    \label{fig:Stab(1)-invariant}
\end{figure}
Now, $f$ can be approached by the following deep neural network by replacing $\rho$ and $\Phi$ by universal approximators as
Diagram \ref{fig:Stab(1)-invariant}. 
We remark that the left part (the part of before taking sum) of this deep neural network is naturally equivariant for the action of $\Stab(1)$.
% : $ \sigma\in \Stab(1), \  \sigma\cdot\begin{pmatrix}
% x_1\\
% \vdots \\
% x_n
% \end{pmatrix} =
% \begin{pmatrix}
% x_{1}\\
% x_{\sigma^{-1}(2)}\\
% \vdots \\
% x_{\sigma^{-1}(n)}
% \end{pmatrix}$.
For an $S_n$-equivariant map $F:\mathbb{R}^n \to \mathbb{R}^n$ with the natural action, by Proposition \ref{prop:relation-invariant-equivariant}, there is a unique $\Stab(1)$-invariant function $f$ such that 
$F(\bm{x})_i=(f\circ(1 \ i))(\bm{x}).$
% \[
% F(\bm{x})_i=(f\circ(1 \ i))(\bm{x}).
% % F(\bm{x})=\begin{pmatrix}
% %     f(\bm{x})\\f\circ(1\ 2)(\bm{x})\\ \vdots \\f\circ(1\ n)(\bm{x})
% %     \end{pmatrix}=
% %     \begin{pmatrix}
% %     f\\f\\ \vdots\\f
% %     \end{pmatrix}
% %     \circ
% %     \begin{pmatrix}
% %     (1\ 1) \\ (1\ 2)\\ \vdots\\ (1\ n)
% %     \end{pmatrix}\circ
% %     \begin{pmatrix}
% %     I_n \\ I_n \\ \vdots\\ I_n 
% %     \end{pmatrix} 
% %     (\bm{x}). 
% \]
Here, $F(\bm{x}) = (F(\bm{x})_1, \dots, F(\bm{x})_n)^\top$ and we regard any element of $S_n$ as a map from $\RR^n$ to $\RR^n$.    
%$\circ$ means the composition of maps, 
%and $I_n$ is the $n\times n$ unit matrix. 
By the argument in this section, %\ref{subsec:approx-Stab(1)-inv}, 
we can approximate $f$ by the previous deep neural network $\{ B_m \circ L \circ A_m\}_m$.
Substituting $B_m \circ L \circ A_m$ for $f$, we construct a deep neural network approximating $F$ as Diagram \ref{fig:Total-DNN}.

The represented function of this neural network of $F_i$ is 
$B_m\circ L \circ A_m \circ (1 \ i)$. 
% \begin{align}
%  B_m\circ L \circ A_m \circ (1 \ i). \label{eq:def-of-DNN}
%     % \begin{pmatrix}
%     % B_m \circ L \circ A_m\\B_m \circ L \circ A_m\\ 
%     % \vdots\\B_m \circ L \circ A_m
%     % \end{pmatrix}
%     % \circ
%     % \begin{pmatrix}
%     % (1\ 1) \\ (1\ 2)\\ \vdots\\ (1\ n)
%     % \end{pmatrix}\circ
%     % \begin{pmatrix}
%     % I_n \\ I_n \\ \vdots\\ I_n 
%     % \end{pmatrix}. 
%     %\label{eq:def-of-DNN}
% \end{align}
%\vspace{-10pt}
The map $F$ splits into two parts, the part of transpositions and part of $(f, f, \dots, f)^\top$. 
On the deep neural network corresponding to $F$ as above, 
the latter part corresponds to the layers from the first hidden layer to the output layer. This part is the $n$ copies of same $\Stab(1)$-invariant deep neural network (an approximation of $(f, f, \dots, f)^\top$). 
Thus, this part is clearly made of equivariant stacking layers for the permutation action of $S_n$.
%, hence is trivially $S_n$-equivariant.
Hence, it is remained to show that the former part is also $S_n$-equivariant.

%By returning it to appropriate part, we can construct the following deep neural network to approximate $F$:

%\todonum{This sentence is correct in grammar?}
We here investigate bounds of the width and the depth of approximators. 
%the width and the depth when we limit either.  
By Theorem~\ref{theo:inv-theorem}, each of $\phi$ and $\rho$  can be approximated by a shallow neural network. Hence, if we do not bound the width, we can obtain deep neural network approximating $F$ with depth three. 
On the other hand, by Theorem~\ref{theo:inv-theorem} again, 
if we do not bound the depth, $\phi$ (resp. $\rho$) can be approximated by 
a deep neural network with width $n+1$ (resp. $n+2$). Thus, we can obtain 
a deep neural network approximating $F$ with width bounded from above by $n^3$. 

% we can obtain a deep neural network approximating  

%\todo{explain each layer and actions on it}

%\subsection{Another $S_n$-action $\ast$ and equivariance}\label{subsec:another-action}

Finally, as step 4, we show that our deep neural network 
is an $S_n$-equivariant deep neural network.  
%The right part of the above deep neural network consisting of $n$ copies of the same deep neural networks (an approximation of $(f, f, \dots, f)^\top$) 
%is clearly made of equivariant stacking layers for the natural action of $S_n$.
%\todonum{not clear}
The most difficult part is to show the equivariance between 
the input layer $\RR^n$ and the first hidden layer $\RR^{n^2}$ 
presented by a function $g\colon\RR^n \to \RR^{n^2}$ as $g=(g_1,\dots, g_n)^\top$ for  
$g_i = \ReLU\circ l \circ (1 \ i)$
% \[ 
% g = 
% \begin{pmatrix}
% \ReLU \\\ReLU \\ \vdots\\ \ReLU
% \end{pmatrix}
% \circ \begin{pmatrix}
% l\\ l \\ \vdots\\l
% \end{pmatrix}
% \circ
% \begin{pmatrix}
% (1 \ 1) \\(1\ 2)\\\vdots\\(1\ n)
% \end{pmatrix}\circ
% \begin{pmatrix}
% I_n \\ I_n \\\vdots\\ I_n
% \end{pmatrix}
% \]
for a certain $\Stab(1)$-invariant linear function 
$l\colon\RR^n \to V$. (Although $V$ is equal to $\RR^n$, we distinguish them to stress the difference.) Unfortunately, the permutation action on the latter space $\RR^{n^2}$ does not make the function $g$ equivariant. For this reason, we need to define another action of $S_n$ on $\RR^{n^2}$ exploiting the $\Stab(1)$-equivariance among each copies. 
%%%%%
% I hesitated over which $\RR^{n^2}$ of $\bigoplus_{i=1}^n V$.
%%%%%

\begin{defn}\label{defn:another-action}
We suppose that $\Stab(1)$ acts on $\RR^n$ by permutation, denoted 
by $\sigma\cdot \bm{x}$ (i.e., by regarding $\Stab(1)$ as a subgroup of $S_n$). Then, we define the action ``$\ast$'' of $S_n$ on $\RR^{n^2}$ as follows: 
\[
% \sigma\ast\begin{pmatrix}
% \bm{x}_1 \\ \vdots\\ \bm{x}_n
% \end{pmatrix}
% =
% \begin{pmatrix}
% \tilde{\sigma}_1\cdot \bm{x}_{\sigma^{-1}(1)} \\ \vdots\\ \tilde{\sigma}_n\cdot \bm{x}_{\sigma^{-1}(n)}
% \end{pmatrix}
\sigma\ast( 
\bm{x}_1, \dots,  \bm{x}_n)
=
( 
\tilde{\sigma}_1\cdot \bm{x}_{\sigma^{-1}(1)}, \dots,  \tilde{\sigma}_n\cdot \bm{x}_{\sigma^{-1}(n)}
)
= (x_{\wt{\sigma}_j^{-1}(i),\sigma^{-1}(j)})_{i,j=1,\dots,n} 
\]
for any $\sigma\in S_n$, and for any 
$(\bm{x}_1, \bm{x}_2, \dots, \bm{x}_n) = (x_{i,j})_{i,j=1,\dots,n} \in (\RR^n)^n$. 
Here, for any $i$, $\tilde{\sigma}_i$ is an element of $\Stab(1)$ defined as  $(1\ i)\sigma=\tilde{\sigma_i}(1\ \sigma^{-1}(i)) $.
\end{defn}

% \begin{prop}
% This action on the left is well defined
% \end{prop}
We will prove the well-definedness of this action ``$\ast$'' in Appendix \ref{sec:appendix-well-definedness}. 
This action is obtained by the injective 
homomorphism 
\begin{align*}
 \varphi\colon S_n \hookrightarrow S_{n^2}; \ \sigma \mapsto \varphi(\sigma), \ \  
  \varphi(\sigma)\cdot(i,j) = (\wt{\sigma}_j^{-1}(i), \sigma^{-1}(j))
\end{align*}
for $(i,j)\in \{1,2, \dots, n \}^2$.  
% {i.e.} for any $\sigma, \tau \in S_n$ and any 
% $X\in \RR^{n^2}$, we have $\tau\ast(\sigma\ast X)=(\tau\sigma)\ast X$. 
% Indeed, this action ``$\ast$'' is based on the coset decomposition of $S_n$ as $\bigsqcup_{i=1}^{n}\Stab(1)(1\ i)$. 
% By this decomposition, $\tilde{\sigma}_i$ is uniquely determined as an element of $\Stab(1)$ for $\sigma$ and $i$. By some  tedious calculation with the use of this relation, we can check the well-definedness of ``$\ast$''.
We remark that this action ``$\ast$'' naturally appears in representation theory as the induced representation, which is an operation to construct a representation of group $S_n$ from a representation of the subgroup $\Stab(1)$ of $S_n$.
We conclude the proof of Theorem \ref{theo:main-theorem} by 
showing the $S_n$-equivariance of $g$:  
%between the input layer $\RR^n and the first hidden layer $\RR^{n^2}$ of our deep neural network.  %(\ref{eq:def-of-DNN}). 
% The situation is as follows: 
% $S_n$ acts on $\RR^n$ and $V$ by 
% the action ``$\cdot$'' induced from permutation, 
% and $\Stab(1)$ acts on $V$ by the same action ``$\cdot$''. 
% The action on $\RR^{n^2}$ is by ``$\ast$'' in Definition \ref{defn:another-action}. 
% In this setting, it suffices to show $S_n$-equivariance of the function $g$ defined as follows: 
% For a $\Stab(1)$-equivariant linear function $l:\mathbb{R}^n\to  V$, the map $g:\RR^n\to \RR^{n^2}$ 
% is defined by 
% \[ 
% g = 
% \begin{pmatrix}
% \ReLU \\\ReLU \\ \vdots\\ \ReLU
% \end{pmatrix}
% \circ \begin{pmatrix}
% l\\ l \\ \vdots\\l
% \end{pmatrix}
% \circ
% \begin{pmatrix}
% (1 \ 1) \\(1\ 2)\\\vdots\\(1\ n)
% \end{pmatrix}\circ
% \begin{pmatrix}
% I_n \\ I_n \\\vdots\\ I_n
% \end{pmatrix}.
% \]
%Then, the following lemma holds:   
\begin{lem}\label{lem:equivariance-first-part}
 The function $g\colon \RR^n \to \RR^{n^2}$ is $S_n$-equivariant. 
\end{lem}
\vspace{-6pt}
%The proof of this lemma is only calculation. 
The detail of proof is in Appendix \ref{subsec:equivariance-of-first-part}.
By this lemma, we conclude the proof of Theorem \ref{theo:main-theorem}.
We remark that the affine transformation $g$ is corresponding to $Z_1$ in the notation (\ref{eq:DNN}). By a representation theoretic aspect, 
this is an intertwining operator between these representation spaces. 
This affine transformation has $n^3$ free parameters a priori. However, by $S_n$-equivariance and a representation theoretic argument, in principle, $g$ can be written by only five parameters. 
By a similar argument, for the other hidden layers, the affine 
transformation $Z_i: \RR^{n^2}\otimes V_i \to \RR^{n^2}\otimes V_{i+1}$ has 
$15\dim V_i\dim V_{i+1}$ parameters (though $n^4\dim V_i\dim V_{i+1}$ a priori).

% then,
% $$=\begin{pmatrix}
% l\\\vdots\\l
% \end{pmatrix}
% \circ
% \begin{pmatrix}
% \tilde{\sigma_1}\cdot(1\ \sigma^{-1}(1)\cdot x\\\tilde{\sigma_2}\cdot(1\ \sigma^{-1}(2)\cdot x\\\vdots\\\tilde{\sigma_n}\cdot(1\ \sigma^{-1}(n)\cdot x
% \end{pmatrix}=
% \begin{pmatrix}
% l\circ (\tilde{\sigma_1}\cdot(1\ \sigma^{-1}(1)\cdot x)\\\vdots\\l\circ (\tilde{\sigma_n}\cdot(1\ \sigma^{-1}(n)\cdot x)
% \end{pmatrix}
% =\begin{pmatrix}
% \tilde{\sigma_1}\cdot l\circ (1\ \sigma^{-1}(1)\cdot x)\\\vdots\\\tilde{\sigma_n}\cdot l\circ (1\ \sigma^{-1}(n)\cdot x)
% \end{pmatrix}$$
% Because $l$ is $\Stab(1)$-equivariant
% $$=\sigma\ast \begin{pmatrix}
% l\circ ((1\ 1)\cdot x)\\\vdots\\l\circ\ ((1\ n)\cdot x)
% \end{pmatrix}=
% \sigma\ast\left(\begin{pmatrix}
% l\\\vdots\\l
% \end{pmatrix}
% \circ
% \begin{pmatrix}
% \tilde{\sigma_1}\cdot(1\ \sigma^{-1}(1))\cdot x\\\tilde{\sigma_2}\cdot(1\ \sigma^{-1}(2))\cdot x\\\vdots\\\tilde{\sigma_n}\cdot(1\ \sigma^{-1}(n))\cdot x
% \end{pmatrix}\right)$$
%Which means the equivariance of $g$.

\section{Dimension reduction}\label{sec:dimension}
In this section, we give the proof of Theorem \ref{thm:number-of-params}.
% we discuss about some advantages of our  invariant/equivariant models. We first point out that our  invariant/equivariant models have fewer parameters than the usual models for a same number of of layers and nodes. This is because they have the constraint conditions induced by the invariance/equivariance with respect to actions of $S_n$. Indeed, the following proposition calculates the number of parameters in our equivariant models.
\begin{prop}\label{layer}
Let $Z_l=\ReLU\circ W_l\colon\mathbb{R}^M \to \mathbb{R}^N$ be an $S_n$-equivariant map.  Assume that the $S_n$-action on $\mathbb{R}^M$ and $\mathbb{R}^N$ is a union of permutations. Then, $n$ divides $M$ and $N$, and the number of the free parameters in $W_l$ is equal to $2MN/n^2$. 
\end{prop}
%\vspace{-6pt}
\begin{proof}
Since $\mathbb{R}^M$ and $\mathbb{R}^N$ have union of permutation actions, by considering the orbit of the cordinates, we see that $n$ devides $M$ and $N$. Let us write $\mathbb{R}^M=(\mathbb{R}^n)^
{M^{\prime}}$ and $\mathbb{R}^N=(\mathbb{R}^n)^{N^{\prime}}$.
In this case, $W_l$ is written by sum of $n\times n$ matrices $V_{ij}$, namely $W_l = (V_{ij})^n_{i,j=1}$.
% \[
%   W_l = \left(
%     \begin{array}{cccc}
%       V_{11} & V_{12} & \ldots & V_{1M'} \\
%       V_{21} & V_{22} & \ldots & V_{2M'} \\
%       \vdots & \vdots & \ddots & \vdots \\
%       V_{M'1} & V_{M'2} & \ldots & V_{ N'M'}
%     \end{array}
%   \right).
% \]
 Here, each $n\times n$ matrix $V_{ij}$ corresponds to the linear map:
\[
\mathbb{R}^n \hookrightarrow (\mathbb{R}^n)^
{M^{\prime}} \xrightarrow{Z_l} (\mathbb{R}^n)^{N^{\prime}} \twoheadrightarrow \mathbb{R}^n, 
\]
where the first map is the inclusion to coordinates of $(\mathbb{R}^n)^
{M^{\prime}}$ and the last map is projection to coordinates of $(\mathbb{R}^n)^{N^{\prime}}$. 
Since these constructions are taken to be compatible with $S_n$-action, we see that $\mbox{ReLU}\circ V_{ij}$ is $S_n$-equivariant. If the activation functions are bijective, we are done because of the same discussion as in the proof of Lemma 3 in \cite{deepsets}. But in our case, we need more discussion, which is in Appendix \ref{subsec:dim}.
%But since ReLU functions are not bijective, we need more discussion.
%Let us take a transposition $(p,q)$. We consider the condition under which the $p$-th coordinate of $V_{ij}(\bm{x})$ is positive. Since the $p$-th coordinate of $V_{ij}(\bm{x})$ is written as an inner product of the $p$-th row vector of $V_{ij}$ and $\bm{x}$, the points in the upper half of the hyperplane defined by $p$-th row vector give positive values. Hence, if $\bm x$ is in the intersection of two hyperplane associated to $p$ and $q$, the $p$-th coordinate and $q$-th coordinate of $V_{ij}(\bm{x})$ is positive. This means that if $\bm x$ is in the intersection of two hyperplane and if we consider the transition $(p, q)$, we may assume $V_{ij}=W_l$. Combining with the discussion in the discussion in the proof of Lemma 3 in \cite{deepsets}, we have
%$$V_{ij}=\lambda \bm I + \gamma (\bm 1 \bm 1^\top).$$
%Hence, for each $n\times n$ matrix, we have two free parameters. Since the number of  $n\times n$ matrices appeared in $W_i$ is $MN/n^2$, we have the desired result.
\end{proof}

\begin{proof}[Proof of Theorem \ref{thm:number-of-params}]
By Proposition \ref{layer}, the number of the free parameter in each equivariant layer is bounded by $2M^2/n^2$. Hence we obtain the desired bound.
\end{proof}

%By this observation, we conclude the number of free parameters in the equivariant models.
%\begin{theo}\label{number}
%Let $\mathcal{N}$ be a deep neural network of the invariant model, consist of the equivariant part of depth $d$ and width $M$ and the preinvariant part of depth $e$ and width $N$ (resp. the equivariant model of depth $d$ and width $M$).
%Then, the number of free parameters in the model is $(2^dM^{2d}/n^{2d})\cdot N^{2e}$ (resp. $2^dM^{2d}/n^{2d}$).
%\end{theo}
%The number of the free parameters in the usual model is $M^{2d}$.
%Hence, Theorem \ref{thm:number-of-params} implies that {\it the free parameters of the invariant/equivarint models are exponentially fewer than the ones of the usual models}.

%Data augmentation is a common technique in empirical learning.
%In the case of the invariant/equivariant tasks, a possible  augmentation is to make new samples by the acting permutation on samples. New samples are effective to the usual models, but not effective to our invariant/equivariant models. This is because in our models, all weights are symmetric under permutation actions. This means that our models learn augmented samples from a sample. By acting permutations, we can make $n!$ new samples from a sample. Hence, computational complexity is reduced to $1/n!$ times. 

\section{Conclusion}
 We introduced some invariant/equivariant models of deep neural networks which are universal approximators for invariant/equivariant functions. The universal approximation theorems in this paper and the discussion in 
 Section \ref{sec:dimension} show that 
 %give a guarantee for the success of the models by \cite{deepsets}, i.e., 
 although the free parameters of our invariant/equivariant models are exponentially fewer than the ones of the usual models, the invariant/equivariant models can approximate the invariant/equivariant functions to arbitrary accuracy.
 Our theory also implies that there is much possibility that the group models behave as the usual models for the tasks related to groups and representation theory can be powerful tool for theory of machine learning. 
 This must be a good perspective to develop the models in deep learning.

%\subsubsection*{Author Contributions}
%If you'd like to, you may include  a section for author contributions as is done
%in many journals. This is optional and at the discretion of the authors.
%
%\subsubsection*{Acknowledgments}
%Use unnumbered third level headings for the acknowledgments. All
%acknowledgments, including those to funding agencies, go at the end of the paper.

%\bibliography{iclr2020_conference}
\bibliographystyle{iclr2020_conference}

\input{iclr2020_conference.bbl}
\newpage
\appendix
%\section{Appendix}

\section{Review of groups and group actions}\label{sec:appendix-groups}
Let $G$ be a set with product, i.e., for any 
$g,h \in G$, $gh$ is defined an element of $G$.  
Then, $G$ is called { a group} if $G$ satisfies 
the following conditions: 
\begin{enumerate}
    \item There is an element $e \in G$ such that $eg=ge = x$ for all $g \in G$. 
    \item For any $g \in G$, there is an element $g^{-1} \in G$ such that $g g^{-1} = g^{-1} g = e$. 
    \item For any $g, h, i \in G$, $(gh)i = g(hi)$ holds. 
\end{enumerate}
Let $G, G'$ be two finite groups. We say a map $\varphi\colon G \to G'$ is  a (group) homomorphism if $\varphi(gh) = \varphi(g)\varphi(h)$. This means that 
the map $\varphi$ preserves the group structures of $G$ in $G'$.

Next, we review actions of groups. Let $X$ be a set. { 
An action of $G$} (or { $G$-action}) on $X$ is defined as a map $G\times X \to X; (g,x) \mapsto g\cdot x$ satisfying the following: 
\begin{enumerate}
    \item  For any $x \in X$, $e\cdot x = x$.  
    \item For any $g,h \in G$ and $x\in X$, $(gh)\cdot x = g\cdot (h\cdot x)$.
\end{enumerate}
%Then, we also say that $G$ acts on $X$ by a left action.
 %A typical example of groups is the permutation group $S_n$ of 
% $n$ elements.  
% The $S_n$ and its stabilizer subgroup $\Stab(1)$ with respect to $\{1 \}$ defined below only appear in this paper: 
%\[
% \Stab_n(1) = \Stab(1) =\left\{\sigma\in S_n \mid \sigma\cdot 1 =1\right\}.  
%\]
%When there is no ambiguity, we denote it by $\Stab(1)$.
% Although $S_n$ acts on the set $X = \{1, 2, \dots, n \}$ by 
% permutation, we remark that 
% we can define some other actions of $S_n$ on $X$, e.g., the  trivial action $\sigma\cdot x = x$ for any $\sigma\in S_n, x \in X$. When we stress the difference of some actions, 
% we use some distinguished letters for each actions as 
 %``$\cdot$'' or ``$\ast$'' etc. In particular, in what follow,  \todonum{English here is strange?}we set the permutation action of $S_n$ by $\sigma(i)$ for $\sigma \in S_n$ and $i \in \{1,2, \dots, n \}$. 

If these conditions are satisfied, we say that $G$ acts on $X$ by a left action. 

\begin{exmp}
 An example which we mainly consider in 
this paper is the permutation group $S_n$ of $n$ elements: 
\[
 S_n = \{ \sigma\colon \{1, \dots, n \} \to 
 \{1, \dots, n \}; \mathrm{bijective} \}
\]
and the product of $\sigma, \tau \in S_n$ is given 
by the composition $\sigma\circ \tau$ as maps. 
$S_n$ acts on the set $\{ 1, 2, \dots, n \}$ by 
the permutation $\sigma\cdot i = \sigma^{-1}(i)$. 
\end{exmp}

% A typical example which we mainly consider in 
% this paper is the permutation group $S_n$ of $n$ elements acting on $X = \{1, 2, \dots, n \}$ by permutation, i.e., $\sigma\cdot i =\sigma(i)$ for $\sigma \in S_n$ and $i \in X$.
We remark that actions of $S_n$ on $X$ is not unique, 
for example, the trivial action $\sigma\cdot x = x$ for any $\sigma\in S_n, x \in X$ is also one of action. 
Hence, when we stress the difference of some actions, 
 we use some distinguished notation for each actions as ``$\cdot$'' or ``$\ast$'' etc. 

Let $G$ be a group and $H$ a subset of $G$. 
We call $H$ a subgroup of $G$ if $H$ is a group with the same product as $G$. 

\begin{exmp}
Let $G$ be a finite group acting on a set $X$. For an element $x\in X$, 
the stabilizer subgroup $\Stab_G(x)$ of G with respects to $\{x \}$ is defined by 
 \[
  \Stab_G(x) = \{ g \in G \mid 
  g\cdot x = x \}. 
 \]
\end{exmp}
When $G= S_n$ and $x = 1$, we use the following notation: $\Stab_n(1)= \Stab(1):=\Stab_G(x)$. 

Let $G, G'$ be two groups. If there is an injective homomorphism  $\varphi\colon G \to G'$, the image $\varphi(G)\subset G'$ can be a subgroup of $G'$. Then, we say that $\varphi$ is an embedding of group $G$ to $G'$. 
Moreover, if $G'$ acts on a set $X$, then $G$ also acts on $G'$ through $\varphi$, i.e., by $\varphi(\sigma)\cdot x$ for $\sigma\in G$ and $x\in X$. 

Then, the following proposition holds: 
\begin{prop}\label{prop:embedding-to-symmetric-group}
Any finite group $G$ can be embedded into $S_n$ for some $n$. 
\end{prop}
\begin{proof}
Let $n:= |G|$ and $G = \{ g_1, g_2, \dots, g_n \}$. For any $g\in G$, 
$gg_i = g_j$ for some $j \in \{ 1,2, \dots, n\}$ as $g g_i \in G$. 
We set $\sigma_g^{-1} (i) = j$. Then, we define 
$\sigma\colon G \to S_n; g \mapsto \sigma_g$. It is easy to show that 
this $\sigma$ is an injective homomorphism. 
\end{proof}

This proposition implies that any finite group $G$ can be realized as 
a permutation action on $R^n$ for some $n$. 

Let $G$ be a finite group acting on $X$. Then, for $x\in X$, we define the 
($G$-)orbit $O_x$ of $x$ as 
\[
 O_x = G\cdot x = \{ g\cdot x \mid g\in G \}. 
\]
Then, for $x, y \in X$, the relation that $x$ and $y$ are in a same orbit 
is an equivalent relation. Hence, $X$ can be divided to a disjoint union of the equivalent classes of this equivalent relation: 
\[
 X = \bigsqcup_{j = 1}^m O_{x_j}. 
\]
We call this the ($G$-)orbit decomposition of $X$. 

Let $H$ be a subgroup of a finite group $G$. Then, for $g \in G$, 
the set 
\[
 gH = \{ gh \in G \mid h \in H\}
\]
is called the left coset of $H$ with respect to $g$. The relation that two elements 
$g$ and $g'$ are in a same coset is also an equivalent relation. 
Hence, we can divide $G$ to a disjoint union of equivalent classes of this 
relation: 
\[
 G = \bigsqcup_{i=1}^l g_iH. 
\]
We call this decomposition the right coset decomposition of $G$ by $H$. 
We set $G/H$ as the set of the left cosets of $G$ by H: 
\[
  G/H = \{ g_1H, g_2H, \dots, g_lH \}.
\]
Then, there is a relation between an orbit and a set of cosets:  
\begin{prop}\label{prop:orbit-stab-bijectivity}
Let $G$ be a finite group acting on a set $X$. For $x\in X$, 
the map 
\[
 G/\Stab_G(x) \to O_x; \ g\Stab_G(x) \to g \cdot x  
\]
is bijective. 
\end{prop}
\begin{proof}
It is easy to check well-definedness and bijectivity.  
\end{proof}

For $G = S_n$ acting on $\{1,2, \dots, n \}$, 
the $G$-orbit of $1$ is only one, i.e.,  $O_1 = \{ 1,2, \dots, n \}$. 
Hence, the following holds (the permutation action of $S_n$ is defined by taking inverse $\sigma^{-1}$, hence we need to consider the set of right cosets): 
\begin{cor}
 The map 
\[
 \Stab_n(1)\!\setminus \! S_n \to \{1,2,\dots, n \}; \Stab_n(1)\sigma \mapsto 
 \sigma^{-1}(1)
\]
is bijective. 
\end{cor}

% To prove Proposition \ref{prop:relation-invariant-equivariant} in 
% Section \ref{sec:equivariant}, we use this relation. 

\section{Proof of Proposition \ref{theo:inv-theorem}}
\label{sec:appendix-inv-theorem}
\begin{proof}[Proof of Proposition \ref{theo:inv-theorem}]
We may assume $N=1$. In fact, 
since we consider the $L^{\infty}$-norm, if all components of $\|f-R_\mathcal{N}\|$ is bounded by $\epsilon$, then $\|f-R_\mathcal{N}\|_\infty \leq \epsilon$ holds.	
For $N=1$, we have $f\colon K\to \mathbb{R}$. Then, by  
Theorem~\ref{thm:representation-of-invariant}, we obtain the representation of $f$ as 
\[
 f(x_1,\dots,x_n)=\rho\left(\sum_{i=1}^{n}\phi(x_i)\right).
\]
By the universal approximation property of ReLU deep neural network, we can find two sequences of ReLU deep neural network  $\{\mathcal{N}^{\rho}_{k}\}_k$ and $\{\mathcal{N}^{\phi}_{k}\}_k$ such that their corresponding functions $\{F_{\mathcal{N}^{\rho}_{k}}\}_k$
and $\{F_{\mathcal{N}^{\phi}_{k}}\}_k$ tend to $\rho$ and $\phi$ for the $L^{\infty}$-norm when $k$ tends to infinity. Let $\{\mathcal{N}_k\}_k$ be the sequence of networks whose corresponding functions are $F_{\mathcal{N}_k} : (x_1,\dots,x_n)\mapsto \left(F_{\mathcal{N}^{\rho}_{k}}\left(\sum_{i=1}^{n}F_{\mathcal{N}^{\phi}_{k}}(x_i)\right) \right)$. To show that $\{F_{\mathcal{N}_k}\}_k$ uniformly tends to $f$, we use the following lemma:
\begin{lem}\label{lem:approximation-composite}
If $\{f_k\}_k$ tends uniformly to $f$, $\{g_k\}_k$ tends uniformly to $g$ and $f$ is uniformly continuous, then $\{f_k\circ g_k\}_k$ tends uniformly to $f\circ g$.
\end{lem}
\begin{proof}[proof of Lemma~\ref{lem:approximation-composite}]
For any $k$, we have 
\begin{align*}
|f_k\circ g_k(x)-f\circ g(x)|&=|f_k\circ g_k(x)-f\circ g_k(x)+f\circ g_k(x)-f\circ g(x)| \\ 
 &\leq \underbrace{|f_k\circ g_k(x)-f\circ g_k(x)|}_{\leq \| f_k-f\|_{\infty}} + |f\circ g_k(x)-f\circ g(x)|.
%& \leq |f_k\circ g_k(x)-f\circ g_k(x)| + |f\circ g_k(x)-f\circ g(x)|. 
\end{align*}
Let $\epsilon > 0$.
By the uniform continuity of $f$, there is a $\delta >0$ 
such that for all $x,y$ satisfying 
$|x-y|\leq \delta$, $|f(x)-f(y)|\leq \epsilon/2$ holds.
Then for large enough $k$, we have both 
$\|g_k-g \|_{\infty}\leq\delta $ which implies that 
for any $x$, $|f\circ g_k(x)-f\circ g(x)|\leq \epsilon/2$, and $||f_k-f||_{\infty}\leq \epsilon/2$.
Hence, for any $k$ large enough, $||f_k\circ g_k-f\circ g||\leq\epsilon$ holds.
\end{proof}
Now using Lemma \ref{lem:approximation-composite}, we have that $(x_1,\dots,x_n)\mapsto\sum_{i=1}^{n}F_{\mathcal{N}^{\phi}_{k}}(x_i)$ tends to $(x_1,\dots,x_n)\mapsto\sum_{i=1}^{n}\phi(x_i)$,  because $(x_1,\dots,x_n)\mapsto\sum_{i=1}^{n}x_i$ is Lipschitz (by triangular inequality)
so uniformly continuous. Then, using Lemma \ref{lem:approximation-composite} again, we obtain the result, since $\rho$ is continuous on a compact set so uniformly continuous.
Moreover, by \cite{sonoda2017neural}, we can approximate $\phi$ and $\rho$ by shallow networks. Hence, we have an approximation by a deep neural network 
$\mathcal{N}$ which has two hidden layers and the width is not bounded.
By \cite{approxrelu}, we can respectively approximate each of $\phi$ and $\rho$ by some neural networks whose width are bounded by $n+2$ and the depth is not bounded. Hence, we have an approximation by a deep neural network $\mathcal{N}$ whose width is bounded by $n(n+2)$ and the depth is not bounded.

The invariant function $f$ is approached by a deep neural network $\mathcal{N}$ having the following diagram:

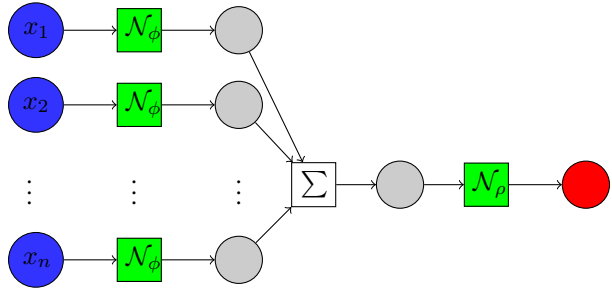
\begin{center}
\begin{tikzpicture}
        \node[functions] (rho) {$\mathcal{N}_{\rho}$};
        
        \node[output,right=2em of rho] (output) {};
            \path[draw,->] (rho) -- (output);
        \node[node,left=1.5em of rho] (a) {}
        node[sum,left=1.5em of a] (sum) {$\sum$};
            \path[draw,->] (sum) -- (a);
            \path[draw,->] (a) -- (rho);
        \node[left=1.5em of sum] (2) {$\vdots$} -- (2) node[left=3em of 2] (l2) {$\vdots$} node[left=3em of l2] (ll2) {$\vdots$}; %node[below of=l2] (ldots) {$\vdots$} node[below of=ll2] (lldots) {$\vdots$} ;
       
        %\node[below of=2] (dots) {$\vdots$}   ;
        \node[node,below of=2] (n) {} -- (n) node[functions,left=2em of n] (ln) {$\mathcal{N}_{\phi}$} node[input,left=2em of ln] (lln) {$x_n$} ;
            \path[draw,->] (lln)--(ln);
            \path[draw,->] (ln) -- (n);
            \path[draw,->] (n) -- (sum);
        \node[node,above =1em of 2] (1) {} -- (1) node[functions,left =2em of 1] (l1) {$\mathcal{N}_{\phi}$}node[input,left=2em of l1] (ll1) {$x_2$} ;
            \path[draw,->] (ll1)--(l1);
            \path[draw,->] (l1) -- (1);
            \path[draw,->] (1) -- (sum);
        \node[node,above =1em of 1] (0) {} -- (0) node[functions,left=2em of 0] (l0) {$\mathcal{N}_{\phi}$}node[input,left=2em of l0] (ll0) {$x_1$} ;
            \path[draw,->] (ll0)--(l0);
            \path[draw,->] (l0) -- (0);
            \path[draw,->] (0) -- (sum);
        
    \end{tikzpicture}
\end{center}
Let us show that this is a $S_n$-{\it invariant} deep neural network. Since  the sum ($\Sigma$) is an invariant function, we can divide this neural network in two parts : the fist part on the left of the symbol $\Sigma$ and the second on the right. For each layer $\mathbb{R}^{d_i}$ of $\mathcal{N}_{\phi}$ there is a corresponding map $Z_i: \mathbb{R}^{d_i} \to \mathbb{R}^{d_{i+1}}$. Then for each layer $(\mathbb{R}^{d_i})^{n}$ of the first part of $\mathcal{N}$, there is a $S_n$ action $ \sigma\cdot (x_i)_i = (x_{\sigma^{-1}(i)})_i$ for 
$\bm{x} = (x_i)_i \in \mathbb{R})^{d_i}$, and the corresponding map $(Z_1,\cdots,Z_n): (\mathbb{R}^{d_i})^n \to (\mathbb{R}^{d_{i+1}})^n$ is $S_n$-equivariant. On the second part of $\mathcal{N}$, there is no $S_n$ actions on the layers $\mathbb{R}^{d_j}$  except the trivial action. Hence, for this action, each corresponding map $Z_i: \mathbb{R}^{d_j} \to \mathbb{R}^{d_{i+j}}$ is invariant. This shows that the network $\mathcal{N}$ is $S_n$-{\it invariant}.
\end{proof}

\section{Proof of Proposition \ref{prop:relation-invariant-equivariant}}
\label{sec:appendix-proof-of-prop}
\begin{proof}[Proof of Proposition 
\ref{prop:relation-invariant-equivariant}]
%\todonum{Revise for any finite group}
First, we show that for $\Stab_G(i_j)$-invariant function $f_j$, 
the map $F = (f_1\circ\sigma_{1,1}, f_1\circ\sigma_{1,2}, \dots, f_m\circ \sigma_{m,l_m})^\top \in \RR^n$ is $G$-equivariant. 
Without loss of generality, we may assume that 
\begin{align*}
 O_{i_1} = \{ 1, 2, \dots,  l_1 \}, \ 
 O_{i_2} = \{ l_1 + 1, l_1 + 2, \dots,  l_1 + l_2 \}, \dots.  
\end{align*} 
In particular, $i_j = \sum_{k = 1}^{j-1} l_k + 1$. For each $j$, we set the coset decomposition by 
\[
 G = \bigsqcup_{k = 1}^{l_j} \Stab_G(i_j) \tau_{j, k}, 
\]
where we may assume that  $\tau_{j,k}$ satisfies 
$\tau_{j,k}^{-1}(i_j) = i_j + k$. 
Now, for $\sigma\in G$, there is a unique $k'$ such that  
$\tau_{j,k}\sigma = \wt{\sigma}_{j,k}\tau_{j,k'}$. Hence, the 
$(ij + k)$-th entry becomes 
\[
 f_j\circ \tau_{j,k}\circ \sigma = f_j\circ \wt{\sigma}_{j,k} \circ 
 \tau_{j,k'} = f_j \circ \tau_{j,k'}. 
\]
On the other hand, we have 
\begin{align*}
 \sigma^{-1}( i_j + k) = \tau_{j,k'}^{-1} \wt{\sigma}_{j,k}^{-1} 
 \tau_{j,k} (i_j + k) = i_j + k'. 
\end{align*}
This shows $F$ is $G$-equivariant. 

Conversely, we assume that $F\colon \RR^n \to \RR^n$ is $G$-equivariant. 
Let $F = (g_1, g_2, \dots, g_n)^\top$. The orbit decomposition 
$O_{i_1}, \dots, O_{i_m}$ is same as above. 
For $\sigma \in G$, by $\sigma\cdot F(\bm{x}) = F(\sigma\cdot \bm{x})$ for any $\bm{x} \in \RR^n$, we have 
\begin{align}
 g_i(\sigma\cdot \bm{x}) = g_{\sigma^{-1}(i)}(\bm{x}).\label{eq:entry-wise}
\end{align}
For simplicity, we consider only for $i = 1$. Then, $\sigma^{-1}(1)\in O_{i_1}$, thus $\sigma^{-1}(1) = 1, \dots, l_1$. 
By the equation (\ref{eq:entry-wise}), for $\sigma \in \Stab_G(1)$, we have 
\[
 g_1(\sigma\cdot \bm{x}) = g_1(\bm{x}),  
\]
as $\sigma \in \Stab_G(1)$ if and only if $\sigma^{-1} \in \Stab_G(1)$. 
Hence, $g_1$ is $G$-invariant. 
The equation (\ref{eq:entry-wise}) for $\tau_{1,k}$ ($k = 1, 2, \dots, l_1$) 
implies 
\[
 g_1(\tau_{1,k}\cdot \bm{x}) = g_{\tau_{1,k}^{-1}}(1)(\bm{x}) =  g_k(\bm{x}). 
\]
This completes the proof for the orbit $O_{i_1}$. 
By same arguments, we can prove the similar result for the other orbits.

\end{proof}
%\fi

\section{Well-definedness of the action ``$\ast$''.}
\label{sec:appendix-well-definedness}

We here show that the left action ``$\ast$'' of $S_n$ on $\RR^{n^2}$ defined in Section \ref{sec:equivariant} is 
well-defined, {i.e.}, for any $\sigma, \tau \in S_n$ and any 
$X = (\bm{x}_1^\top, \dots, \bm{x}_n^\top)^\top \in \RR^{n^2}$, we have $\tau\ast(\sigma\ast X)=(\tau\sigma)\ast X$.

% We need to check that $\forall \sigma, \tau \in S_n$, $\forall X\in \oplus_{i=1}^{n}V_i$, we have $\tau\ast(\sigma\ast X)=(\tau\sigma)\ast X$.\\

%We have $S_n=\bigsqcup_{i=1}^{n}\Stab(1)(1\ i)$\\ 
%\\

The permutation group $S_n$ is decomposed as 
\[
 S_n = \bigsqcup_{i=1}^n \Stab(1) (1\ i). 
\]
For any $\sigma\in S_n$, because $\sigma(1\ \sigma^{-1}(1))$ is in $\Stab(1)$, 
this $\sigma$ is in the coset $\Stab(1)(1 \ \sigma^{-1}(1))$.
Apply this for $(1\ i)\sigma$ for $i$,  
$(1 \ i)\sigma$ is in the coset  
\[
 \Stab(1)(1\ ((1\ i)\sigma)^{-1}(1)) = 
 \Stab(1)(1\ \sigma^{-1}(i)),  
\]
thus, we have a unique element 
$\tilde{\sigma_i} \in \Stab(1)$ such that 
\begin{align}\label{eq:sigma-tilde}
 (1\ i)\sigma = \tilde{\sigma}_i (1\ \sigma^{-1}(i)).
\end{align}
For $\sigma, \tau \in S_n$ and 
$X= (\bm{x}_1^\top, \dots, \bm{x}_n^\top)^\top\in \RR^{n^2}$, 
we have 
\begin{align*}
\tau\ast(\sigma\ast X)
&=\tau\ast \begin{pmatrix}
\tilde{\sigma}_1\cdot x_{\sigma^{-1}(1)} \\ \vdots\\ \tilde{\sigma}_n\cdot x_{\sigma^{-1}(n)}
\end{pmatrix}
=\begin{pmatrix}
\tilde{\tau_1}\tilde{\sigma}_{\tau^{-1}(1)}\cdot x_{\sigma^{-1}(\tau^{-1}(1))} \\ \vdots\\ \tilde{\tau_1}\tilde{\sigma}_{\tau^{-1}(n)}\cdot x_{\sigma^{-1}(\tau^{-1}(n))}
\end{pmatrix}\\ 
&=\begin{pmatrix}
\tilde{\tau_1}\tilde{\sigma}_{\tau^{-1}(1)}\cdot x_{(\tau\sigma)^{-1}(1)} \\ \vdots\\ \tilde{\tau_1}\tilde{\sigma}_{\tau^{-1}(n)}\cdot x_{(\tau\sigma)^{-1}(n)}
\end{pmatrix}
\end{align*}
Then, by the equation (\ref{eq:sigma-tilde}), 
$\tilde{\sigma}_i = (1\ i) \sigma (1\ \sigma^{-1}(i))$. 
Hence we have 
\begin{align*}
\tilde{\tau_i}\tilde{\sigma}_{\tau^{-1}(i)}
&=(1\ i)\tau(1\ \tau^{-1}(i))(1\ \tau^{-1}(i))\sigma(1\ \sigma^{-1}(\tau^{-1}(i))) \\ 
&=(1\ i)\tau\sigma(1\ (\tau\sigma)^{-1}(i))
=\wt{(\tau\sigma)}_i.
\end{align*}
This relation implies 
\begin{align*}
\tau\ast(\sigma\ast X)=\begin{pmatrix}
\tilde{\tau_1}\tilde{\sigma}_{\tau^{-1}(1)}\cdot x_{(\tau\sigma)^{-1}(1)} \\ \vdots\\ \tilde{\tau_1}\tilde{\sigma}_{\tau^{-1}(n)}\cdot x_{(\tau\sigma)^{-1}(n)}
\end{pmatrix}
=\begin{pmatrix}
\widetilde{(\tau\sigma)}_1\cdot x_{(\tau\sigma)^{-1}(1)} \\ \vdots\\ \widetilde{(\tau\sigma)}_n\cdot x_{(\tau\sigma)^{-1}(n)}
\end{pmatrix}=(\tau\sigma)\ast X
\end{align*}
Thus, the action is well-defined.

To satisfy the hypothesis of Theorem 2.3, we need to check that the "$\ast$" action is free.

\section{Proof of Lemma \ref{lem:equivariance-first-part}}
\label{subsec:equivariance-of-first-part}
\begin{proof}[Proof of Lemma   \ref{lem:equivariance-first-part}]
For any $i$, any $\sigma\in S_n$ and any  $\bm{x}\in\mathbb{R}^{n}$, we have
\[
 (l\circ (1 \ i) \circ I_n)(\sigma\cdot \bm{x}) = 
 l (((1 \ i)\sigma) \cdot \bm{x}). 
\]
% $$\begin{pmatrix}
% l\\\vdots\\l
% \end{pmatrix}
% \circ
% \begin{pmatrix}
% id\\(1\ 2)\\\vdots\\(1\ n)
% \end{pmatrix} 
% (\sigma\cdot x)=\begin{pmatrix}
% l\\\vdots\\l
% \end{pmatrix}
% \circ
% \begin{pmatrix}
% \sigma\cdot x\\(1\ 2)\sigma\cdot x\\\vdots\\(1\ n)\sigma\cdot x
% \end{pmatrix}$$

Because for any $i$, there is a unique $\tilde{\sigma}_i \in \Stab(1)$ such that  $(1\ i)\sigma=\tilde{\sigma_i}(1\ \sigma^{-1}(i))$ as 
in Definition \ref{defn:another-action}, we have 
\[
 l ( ((1 \ i)\sigma) \cdot \bm{x}) = 
 l (\tilde{\sigma_i}(1\ \sigma^{-1}(i)) \cdot \bm{x}) 
 = \tilde{\sigma_i} \cdot l ((1\ \sigma^{-1}(i)) \cdot \bm{x}). 
\] 
The last equality is due to $\Stab(1)$-equivariance of $l$.
On the other hand, by Definition \ref{defn:another-action}, 
$i$-th entry of $\sigma\ast g(\bm{x})$ becomes 
\begin{align*}
  (\sigma\ast g(\bm{x}))_i &= \tilde{\sigma}_i \cdot g(\bm{x})_{\sigma^{-1}(i)} \\
  & = \tilde{\sigma}_i\cdot (\ReLU ( l ((1 \ \sigma^{-1}(i))\cdot \bm{x}))). 
\end{align*}
Because $\tilde{\sigma}_i \circ \ReLU = \ReLU \circ\tilde{\sigma}_i$ holds, $g$ is $S_n$-equivariant. 
\end{proof}

\section{Data augmentation}
Data augmentation is a common technique in empirical learning.
In the case of the invariant/equivariant tasks, a possible  augmentation is to make new samples by the acting permutation on samples. New samples are effective to the usual models, but not effective to our invariant/equivariant models. This is because in our models, all weights are symmetric under permutation actions. This means that our models learn augmented samples from a sample. By acting permutations, we can make $n!$ new samples from a sample. Hence, computational complexity is reduced to $1/n!$ times. 

\section{Dimension reduction}\label{subsec:dim}
In this section, we give the proof of Proposition \ref{layer}.

\begin{proof}[Proof of Proposition \ref{layer}]
%We firstly see that $n$ devides $M$ and $N$. By induction on $L$, we claim that if $\mathbb{R}^L$ has $S_n$-action, $n$ devides $L$. 
%If $L$ is smaller than $n$, it is easy to see that $\mathbb{R}^L$ does not have $S_n$-action.  Let $(x_i)_{i\in I}$ be the coordinate of $\mathbb{R}^L$. Since $\mathbb{R}^L$ have $S_n$-action, for each $\sigma \in S_n$, there is a $j_{\sigma}$ such that $\sigma \circ x_1 = x_{j_{\sigma}}$ holds. We can see that the set $J_1=\{j_{\sigma}|\sigma \in S_n\}$ is consist of $n$ elements. Let $I_1=I-J_1$, then $(x_i)_{i\in I_1}$  is the coordinate of $\mathbb{R}^{L-n}$ and $\mathbb{R}^{L-n}$ has $S_n$-action induced by the one of $\mathbb{R}^L$.  By the inductive hypothesis, $n$ devides $L-n$, hence $n$ devides $L$.
Since the action on $\mathbb{R}^M, \mathbb{R}^N$ is a union of permutations on nodes in each equivariant layer,
we can write $\mathbb{R}^M=(\mathbb{R}^n)^
{M^{\prime}}$ and $\mathbb{R}^N=(\mathbb{R}^n)^{N^{\prime}}$.
In this case, $W_i$  can be written as the following form; 
  \[
   W_i = \left(
     \begin{array}{cccc}
       V_{11} & V_{12} & \ldots & V_{1M'} \\
       V_{21} & V_{22} & \ldots & V_{2M'} \\
       \vdots & \vdots & \ddots & \vdots \\
V_{M'1} & V_{M'2} & \ldots & V_{ N'M'}
           \end{array}
   \right)
 \]

, where $Vij$ are $n\times n$ matrices.
Let us consider the following maps: $$\mathbb{R}^n \xrightarrow{i} (\mathbb{R}^n)^
{M^{\prime}} \xrightarrow{Z_i} (\mathbb{R}^n)^{N^{\prime}} \xrightarrow{p} \mathbb{R}^n \\
$$
, where the first map is the inclusion to the coordinates started from the $(j-1)n$ th coordinate of $(\mathbb{R}^n)^
{M^{\prime}}$ ended at the $jn-1$ th coordinate of $(\mathbb{R}^n)^
{M^{\prime}}$ and the last map is the projection to the coordinates started from the $(i-1)n$ th coordinate of $(\mathbb{R}^n)^
{N^{\prime}}$ ended at the $in-1$ th coordinate of $(\mathbb{R}^n)^
{N^{\prime}}$ .  We chase the elements of these maps as follows;
\begin{align*}
p\circ Z_i \circ i \left(\left[
    \begin{array}{c}
      a_1 \\
      a_2 \\
      \vdots \\
      a_n
    \end{array}
  \right]\right)
  &=p\circ Z_i 
   \left(\left[
    \begin{array}{c}
      0 \\
      \vdots \\
      0\\
      a_1 \\
      a_2 \\
      \vdots \\
      a_n\\
      0 \\
      \vdots \\
      0
    \end{array}
  \right]
  \right)
  =p \left( \mbox{ReLU}\circ \left(
    \begin{array}{cccc}
      V_{11} & V_{12} & \ldots & V_{1M'} \\
      V_{21} & V_{22} & \ldots & V_{2M'} \\
      \vdots & \vdots & \ddots & \vdots \\
      V_{N'1} & V_{N'2} & \ldots & V_{ N'M'}
    \end{array}
  \right)
  \left[
    \begin{array}{c}
      0 \\
      \vdots \\
      0\\
      a_1 \\
      a_2 \\
      \vdots \\
      a_n\\
      0 \\
      \vdots \\
      0
    \end{array}
  \right]
  \right)\\
  &=p\left(\begin{array}{c}
  \mbox{ReLU}\circ V_{1j}
  \left[
  \begin{array}{c}
      a_1 \\
      a_2 \\
      \vdots \\
      a_n\\
    \end{array}
  \right]
  \\
  \vdots\\
    \mbox{ReLU}\circ V_{M'j}
  \left[
  \begin{array}{c}
      a_1 \\
      a_2 \\
      \vdots \\
      a_n\\
    \end{array}
  \right]
  \end{array}
    \right)
    =
  \mbox{ReLU}\circ V_{ij}
  \left[
  \begin{array}{c}
      a_1 \\
      a_2 \\
      \vdots \\
      a_n\\
    \end{array}
  \right]
\end{align*}
Hence, each $n\times n$ matrix $V_{ij}$ induces a subneural network.
Since these constructions are taken to be compatible with $S_n$-action, we see that $f=\mbox{ReLU}\circ V_{ij}: \mathbb{R^n} \to \mathbb{R^n}$ is $S_n$-equivariant. If the activation functions are bijective, we are done because of the same discussion as in the proof of Lemma 3 in \cite{deepsets}. But since ReLU functions are not bijective, we need more discussion.
Let us take a transposition $\sigma=(p \ q)$. Since $f$ is $S_n$ -equivariant, $\sigma\circ f(\bm{x}) = f(\sigma\circ\bm{x})$ holds for any $\bm{x}$.  We have

\begin{align*}
%\begin{eqnarray*}
\sigma\circ f(\bm{x})
&= \sigma \left( \mbox{ReLU}\circ  \left(
    \begin{array}{ccccccc}
      a_{11} &\ldots & a_{1n}\\
      \vdots & \ddots & \vdots\\
      a_{p1} &\ldots  & a_{pn} \\
      \vdots & \ddots  & \vdots \\
      a_{q1} &\ldots & a_{qn}\\
      \vdots &\ddots  & \vdots \\
      a_{n1} & \ldots & a_{nn}
    \end{array}
  \right)
  \left[
    \begin{array}{c}
      x_1 \\
      x_2 \\
      \vdots \\
      x_n\\
     \end{array}
  \right]
  \right)
  =\mbox{ReLU} \left( \sigma\circ \left(  
    \begin{array}{ccccccc}
      a_{11} &\ldots & a_{1n}\\
      \vdots & \ddots & \vdots\\
      a_{p1} &\ldots  & a_{pn} \\
      \vdots & \ddots  & \vdots \\
      a_{q1} &\ldots & a_{qn}\\
      \vdots &\ddots  & \vdots \\
      a_{n1} & \ldots & a_{nn}
    \end{array}
  \right)
  \left[
    \begin{array}{c}
      x_1 \\
      x_2 \\
      \vdots \\
      x_n\\
     \end{array}
  \right]
  \right)\\
  &=\mbox{ReLU}\left( \sigma \circ
   \left[ \begin{array}{c}
      \sum_{j=1}^n a_{1j}x_j \\
      \vdots\\
      \sum_{j=1}^n a_{pj}x_j  \\
      \vdots \\
       \sum_{j=1}^n a_{qj}x_j  \\
      \vdots \\
     \sum_{j=1}^n a_{nj}x_j \\
     \end{array}
     \right]
  \right)
  =\mbox{ReLU}\left( 
   \left[ \begin{array}{c}
      \sum_{j=1}^n a_{1j}x_j \\
      \vdots\\
      \sum_{j=1}^n a_{qj}x_j  \\
      \vdots \\
       \sum_{j=1}^n a_{pj}x_j  \\
      \vdots \\
     \sum_{j=1}^n a_{nj}x_j \\
     \end{array}
     \right]
  \right)
=  \mbox{ReLU} \left(  \left(
    \begin{array}{ccccccc}
      a_{11} &\ldots & a_{1n}\\
      \vdots & \ddots & \vdots\\
      a_{q1} &\ldots  & a_{qn} \\
      \vdots & \ddots  & \vdots \\
      a_{p1} &\ldots & a_{pn}\\
      \vdots &\ddots  & \vdots \\
      a_{n1} & \ldots & a_{nn}
    \end{array}
  \right)
  \left[
    \begin{array}{c}
      x_1 \\
      x_2 \\
      \vdots \\
      x_n\\
     \end{array}
  \right]
  \right).
%\end{eqnarray*}
\end{align*}
On the other hand, we have
\begin{align*}
f(\sigma \circ \bm{x}) = f\left( \left[
    \begin{array}{c}
      x_1 \\
      \vdots \\
      x_q \\
      \vdots \\
      x_p \\
      \vdots\\
      x_n\\
     \end{array}
  \right]
  \right)
  &= \mbox{ReLU} \left( \left(
    \begin{array}{ccccccc}
      a_{11} &\ldots & a_{1p} &\ldots & a_{1q} & \ldots & a_{1n} \\
      a_{21} &\ldots & a_{2p} &\ldots & a_{2q} & \ldots & a_{2n} \\
      \vdots & \ddots  & \vdots  & \ddots & \vdots & \ddots  & \vdots \\
      a_{n1} &\ldots & a_{np} & \ldots & a_{nq} & \ldots & a_{nn}
    \end{array}
  \right)
  \left[
    \begin{array}{c}
      x_1 \\
      \vdots \\
      x_q \\
      \vdots \\
      x_p \\
      \vdots\\
      x_n\\
     \end{array}
  \right]
  \right)\\
  &= \mbox{ReLU} \left( \left(
    \begin{array}{ccccccc}
      a_{11} &\ldots & a_{1q} &\ldots & a_{1p} & \ldots & a_{1n} \\
      a_{21} &\ldots & a_{2q} &\ldots & a_{2p} & \ldots & a_{2n} \\
      \vdots & \ddots  & \vdots  & \ddots & \vdots & \ddots  & \vdots \\
      a_{n1} &\ldots & a_{nq} & \ldots & a_{np} & \ldots & a_{nn}
    \end{array}
  \right)
  \left[
    \begin{array}{c}
      x_1 \\
      \vdots \\
      x_p \\
      \vdots \\
      x_q \\
      \vdots\\
      x_n\\
     \end{array}
  \right]
  \right).
\end{align*}

\begin{claim}
$ \left(
    \begin{array}{ccccccc}
      a_{11} &\ldots & a_{1q} &\ldots & a_{1p} & \ldots & a_{1n} \\
      a_{21} &\ldots & a_{2q} &\ldots & a_{2p} & \ldots & a_{2n} \\
      \vdots & \ddots  & \vdots  & \ddots & \vdots & \ddots  & \vdots \\
      a_{n1} &\ldots & a_{nq} & \ldots & a_{np} & \ldots & a_{nn}
    \end{array}
  \right)
  =
  \left(
    \begin{array}{ccccccc}
      a_{11} &\ldots & a_{1n}\\
      \vdots & \ddots & \vdots\\
      a_{q1} &\ldots  & a_{qn} \\
      \vdots & \ddots  & \vdots \\
      a_{p1} &\ldots & a_{pn}\\
      \vdots &\ddots  & \vdots \\
      a_{n1} & \ldots & a_{nn}
    \end{array}
  \right)
$
\end{claim}
\begin{proof}[Proof of Claim 1]
We have
\begin{align*}
&\mbox{ReLU} \left(  \left(
    \begin{array}{ccccccc}
      a_{11} &\ldots & a_{1n}\\
      \vdots & \ddots & \vdots\\
      a_{q1} &\ldots  & a_{qn} \\
      \vdots & \ddots  & \vdots \\
      a_{p1} &\ldots & a_{pn}\\
      \vdots &\ddots  & \vdots \\
      a_{n1} & \ldots & a_{nn}
    \end{array}
  \right)
  \left[
    \begin{array}{c}
      x_1 \\
      x_2 \\
      \vdots \\
      x_n\\
     \end{array}
  \right]
  \right)=  
&\mbox{ReLU} \left( \left(
    \begin{array}{ccccccc}
      a_{11} &\ldots & a_{1q} &\ldots & a_{1p} & \ldots & a_{1n} \\
      a_{21} &\ldots & a_{2q} &\ldots & a_{2p} & \ldots & a_{2n} \\
      \vdots & \ddots  & \vdots  & \ddots & \vdots & \ddots  & \vdots \\
      a_{n1} &\ldots & a_{nq} & \ldots & a_{np} & \ldots & a_{nn}
    \end{array}
  \right)
  \left[
    \begin{array}{c}
      x_1 \\
      \vdots \\
      x_p \\
      \vdots \\
      x_q \\
      \vdots\\
      x_n\\
     \end{array}
  \right]
  \right)
\end{align*}
  for any $\bm{x}$.
  Hence, the the $l$ -th coordinate of the left hand side is positive if and only if the one of the right hand side is positive.It is clear that each equation is positive for infinitely many $\bm{x}$. This implies that the coefficients of each equations coincide. Hence, we have the desired result.
\end{proof}
We show that $a_{pp}=a_{qq}$ holds for any $p,q$.
We can see that the $(p \ q)$ entry of the matrix of the left hand side in the claim is equal to $a_{pp}$.  Similarly, the 
$(p \ q)$ entry of the matrix of the right hand side in the claim is equal to $a_{qq}$. Hence, by the claim, we have $a_{pp}=a_{qq}$.
We show that if $i \neq j$ and $s\neq t$, $a_{ij}=a_{st}$ holds. Consider the $(i \ q)$ entry of each matrices, the one of the left hand side is equal to $a_{ip}$ and the one of the right hand side is equal to $a_{iq}$.
Hence, we have $a_{ip}=a_{iq}$, where $i \neq p$ ane $i \neq q$. By the symmetry,  $a_{pi}=a_{qi}$ holds for any $i\neq p$ and $i\neq q$.
Hence, we have
$$
a_{ij}=a_{sj}=a_{st}
$$
for any $i \neq j$ and $s \neq t$.
Hence, we can write  $V_{ij}=\lambda \bm I + \gamma (\bm 1 \bm 1^\top)$.
\end{proof}

\end{document}

%% file: math_commands.tex
\usepackage{amsmath,amsfonts,bm}

\def\eqref#1{equation~\ref{#1}}
\def\1{\bm{1}}

\DeclareMathAlphabet{\mathsfit}{\encodingdefault}{\sfdefault}{m}{sl}
\SetMathAlphabet{\mathsfit}{bold}{\encodingdefault}{\sfdefault}{bx}{n}

\usepackage[utf8]{inputenc} % allow utf-8 input
\usepackage[T1]{fontenc}    % use 8-bit T1 fonts
\usepackage{hyperref}       % hyperlinks
\usepackage{url}            % simple URL typesetting
\usepackage{booktabs}       % professional-quality tables
\usepackage{amsfonts}       % blackboard math symbols
\usepackage{nicefrac}       % compact symbols for 1/2, etc.
\usepackage{microtype}      % microtypography

\usepackage[textsize=small]{todonotes}
\newcounter{todocounter}

\newcounter{mycomment}

\usepackage{amsmath, amsfonts, amssymb, amsthm, array}
\usepackage{mathtools, bm}
\usepackage{color}
\usepackage{xcolor}

\usepackage{cases}
\usepackage{xr}
\usepackage{tikz}
\usetikzlibrary{positioning}
\theoremstyle{plain}
\newtheorem{theo}{Theorem}[section]
\newtheorem{claim}{Claim}
\newtheorem{lem}{Lemma}[section]
\newtheorem{prop}{Proposition}[section]
\newtheorem{cor}{Corollary}[section]
\theoremstyle{definition}
\newtheorem{defn}{Definition}[section]
\newtheorem{exmp}{Example}[section]

\tikzset{basic/.style={draw,fill=blue!80,text width=1em,text badly centered}}
\tikzset{sum/.style={basic,rectangle,fill=white}}
\definecolor{mygray}{gray}{0.8}
\tikzset{node/.style={basic,circle,fill=mygray}}
\tikzset{input/.style={basic,circle}}
\tikzset{output/.style={basic,circle,fill=red}}
\tikzset{weights/.style={basic,rectangle}}
\tikzset{functions/.style={basic,rectangle,fill=green}}

\newcounter{num}
\newcommand{\RR}{\mathbb{R}} % real number field
\newcommand{\wt}[1]{\widetilde{#1}} % widetilde
\newcommand{\Stab}{\mathrm{Stab}}
\newcommand{\ReLU}{\mathrm{ReLU}}
\newcommand{\inv}{\mathrm{inv}}
\newcommand{\equi}{\mathrm{equiv}}